\documentclass[preprint,authoryear,10pt,a4paper]{elsarticle}
\usepackage[top=0.75in, bottom=0.75in, left=0.75in, right=0.75in]{geometry}
\usepackage{palatino}
\usepackage{amsmath,amsthm}
\usepackage{amssymb}
\usepackage{bbm}
\usepackage{amsfonts}
\usepackage{graphicx}
\usepackage{subfigure}
\usepackage{grffile}
\graphicspath{{graphics/}}
\usepackage[ruled,vlined,linesnumbered]{algorithm2e}
\usepackage{tabularx,booktabs}
\usepackage{inputenc} 
\usepackage[T1]{fontenc}   
\usepackage{url}            % simple URL 
\usepackage{booktabs}       % 
\usepackage{amsmath,amsthm}
\usepackage{amssymb}
\usepackage{bbm}
\usepackage{amsfonts}       % blackboard math 
\usepackage{nicefrac}       % compact symbols 
\usepackage{microtype}      % microtypography
\usepackage{tikz}
\usepackage{rotfloat}
\usepackage{subfigure}
\usepackage{graphicx}
\usepackage{caption}
\usepackage{appendix}
\usepackage{tablefootnote}
\usepackage{bm}
\usepackage{mathrsfs}
\usepackage{xcolor}
\usepackage{colortbl}

\usepackage{verbatim}
\usepackage{xcolor}
\definecolor{darkblue}{rgb}{0.0,0.5,0.5}
\definecolor{blue}{rgb}{0.0,0.59,0.84}
\definecolor{myblue}{RGB}{0,0,255}

\usepackage[colorlinks]{hyperref}
\hypersetup{colorlinks,breaklinks,linkcolor=blue,urlcolor=blue,anchorcolor=blue,citecolor=blue}
\usepackage{booktabs,caption}
\usepackage{multirow,bigdelim}
\usepackage{lscape}
\usepackage{lineno}
\usepackage{microtype}
\newtheorem{lemma}{Lemma}

\usepackage{soul}
\usepackage{lineno}
\usepackage{fancyhdr}
\newcommand{\minew}[1]{{\color{myblue}{#1}}}

\journal{Transportation Research Part C: Emerging Technologies}

\begin{document}

\begin{frontmatter}



% \title{{\fontfamily{lmss}\selectfont Spatiotemporal graph embedded low-rank tensor learning for network-wide traffic speed kriging with incomplete observations}}
% \title{{\fontfamily{lmss}\selectfont  Spatiotemporal traffic speed kriging for enhancing network-wide sensor perception with low coverage: an integrated graph tensor learning method}}

% \title{{\fontfamily{lmss}\selectfont  Augmenting spatial coverage of traffic volume sensors for highway network via spatiotemporal correlation adaptive graph neural networks}}
% \title{{\fontfamily{lmss}\selectfont Towards better traffic volume estimation: Tackling both underdetermination and nonequilibrium problems via correlation-adaptive graph convolutional networks}}
% \title{{\fontfamily{lmss}\selectfont Generalizable Implicit Neural Representations As a Universal Spatiotemporal Traffic Data Learner}}

\title{{\fontfamily{lmss}\selectfont \textbf{Spatiotemporal Implicit Neural Representation as a Generalized Traffic Data Learner}}}

% \title{{\fontfamily{lmss}\selectfont \textbf{Spatiotemporal Implicit Neural Representation as a Generalized Approach for Learning Traffic Data}}}

\author[label,label1]{Tong Nie}
\author[label]{Guoyang Qin}
\author[label1]{Wei Ma\corref{cor1}}
\ead{wei.w.ma@polyu.edu.hk}
\author[label]{Jian Sun\corref{cor1}}
\ead{sunjian@tongji.edu.cn}

\address[label]{Department of Traffic Engineering, Tongji University, Shanghai, 201804, China}
\address[label1]{Department of Civil and Environmental Engineering, The Hong Kong Polytechnic University, Hong Kong SAR, China}

\cortext[cor1]{Corresponding authors.}
% \cortext[cor2]{Corresponding author. Address: Cao’an Road 
% 4800, Shanghai, 201804, China}

\begin{abstract}
Spatiotemporal Traffic Data (STTD) measures the complex dynamical behaviors of the multiscale transportation system. 
Existing methods aim to reconstruct STTD using low-dimensional models. 
However, they are limited to data-specific dimensions or source-dependent patterns, restricting them from unifying representations. 
Here, we present a novel paradigm to address the STTD learning problem by parameterizing STTD as an implicit neural representation. 
To discern the underlying dynamics in low-dimensional regimes, coordinate-based neural networks that can encode high-frequency structures are employed to directly map coordinates to traffic variables.
To unravel the entangled spatial-temporal interactions, the variability is decomposed into separate processes.
We further enable modeling in irregular spaces such as sensor graphs using spectral embedding. 
Through continuous representations, our approach enables the modeling of a variety of STTD with a unified input, thereby serving as a generalized learner of the underlying traffic dynamics. 
% We also indicate that it possesses properties such as implicit low-rankness, inherent smoothness, and reduced complexity, making it versatile for a wide range of practical applications.
It is also shown that it can learn implicit low-rank priors and smoothness regularization from the data, making it versatile for learning different dominating data patterns.
We validate its effectiveness through extensive experiments in real-world scenarios, showcasing applications from corridor to network scales. 
Empirical results not only indicate that our model has significant superiority over conventional low-rank models, but also highlight that the versatility of the approach extends to different data domains, output resolutions, and network topologies. 
Comprehensive model analyses provide further insight into the inductive bias of STTD.
We anticipate that this pioneering modeling perspective could lay the foundation for universal representation of STTD in various real-world tasks. 

% Comprehensive model analyzes further validate its mechanism and provide insight into the inductive bias of STTD.
% We believe that this pioneering modeling perspective will lay the foundation for universal STTD learning, offering key properties essential for diverse real-world tasks. 
% \textbf{The project page with the available code will be released at}: \url{https://github.com/tongnie/traffic_dynamics}.
\end{abstract}


\end{frontmatter}

% \textcolor{brown}{\textbf{THERE NEEDS A FASCINATING TEASER FIGURE.}}
\begin{figure}[!htb]
\centering
\includegraphics[width=0.99\textwidth]{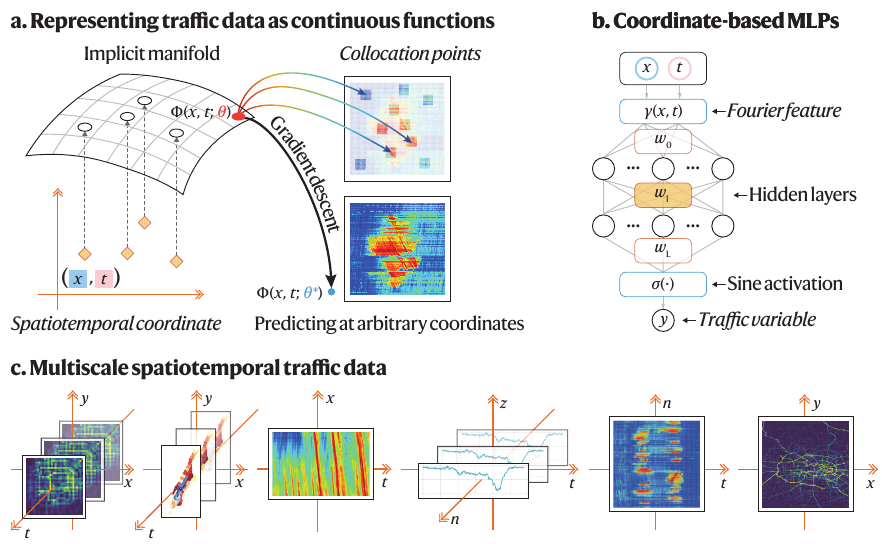}
\caption{\textbf{Representing spatiotemporal traffic data as an implicit neural function.} \textbf{(a)} Traffic data at arbitrary spatial-temporal 
 coordinates can be represented as a continuous function in an implicit space. \textbf{(b)} Coordinate-based MLPs map input coordinates to traffic state of interest. \textbf{(c)} With the resolution-independent property, our model can learn a variety of spatiotemporal traffic data from different sources.}
\label{fig:intro}
\end{figure}

% \linenumbers
\section{Introduction}\label{Introduction}
The heterogeneous elements involved in a vehicular traffic system, such as travel demand, human behaviors, infrastructure supply, weather conditions, and social economics, lead to a complex, high-dimensional, and large-scale dynamical system \citep{avila2020data}. 
% Despite being complicated in system details, the evolution of the system can be accessible. 
To better understand this system, spatiotemporal traffic data (STTD) is a quantity that is often collected to describe its evolution in space and time \citep{yuan2021survey}.
% Genrerally, STTD refers to a group of correlated traffic data records that span a time period within a certain space, which is measured either by Eulerian measurements at fixed locations or Lagrangian measurements from mobile sensors.
This data includes various sources such as vehicle trajectories, sensor-based time series, and dynamic mobility flow, which is measured either by Eulerian measurements at fixed locations or Lagrangian measurements from mobile sensors. 
Traffic participants are constantly immersed in STTD, perceiving it through the daily travels of moving agents, recording it through different types of traffic detectors, and finally analyzing and utilizing it.

In practical traffic engineering, STTD is one of the main ingredients required to provide vital information for the timely and effective deployment of large-scale traffic control strategies \citep{tsitsokas2023two,hu2024demonstration}, data-driven traffic demand management \citep{nallaperuma2019online}, and elaborate traffic optimization routines \citep{wang2024traffic} in modern intelligent transportation systems \citep{zhang2011data}. 
Meanwhile, the ever-increasing amount of STTD in transportation systems has left traffic agencies in need of a generalized method to analyze the continuously collected data in various types.
However, one can face the challenge of fully understanding the mystery of STTD, which can be caused by multiple interwoven factors. 
For example, traffic congestion dynamics involves entangled relationships between spatial and temporal domains \citep{bellocchi2020unraveling,saberi2020simple,duan2023spatiotemporal}, and the high dimensionality of STTD further intensifies the challenges in data processing \citep{asif2013spatiotemporal}.
These difficulties undermine its direct utility in reality.

STTD modeling has gained greater attention in the field of traffic data science and engineering in recent years.
To dissect the spatiotemporal complexity, 
% a typical objective in STTD learning is to use the measured noisy data to construct data-centric models that can then be used to predict the complex behaviors of the underlying systems. 
the primary aim of STTD learning is to develop data-centric \textit{learners} that can accurately \textit{learn predictive functions from observations and predict complex dynamics of STTD}.
In this sense, the system that generates STTD can be characterized as a nonlinear dynamical system that exhibits abundant multiscale and interactive phenomena. Despite its complexity, 
% the dynamics of the system evolves with some dominating patterns that can be captured by some low-dimensional structures \citep{thibeault2024low,wu2024predicting}. 
recent advances in STTD have found that the dynamics of the system evolve with some dominating patterns and can be captured by some low-dimensional structures \citep{thibeault2024low,wu2024predicting}.
% One of the representative low-dimensional methods is \textit{low-rank} model. This kind of model utilizes the algebraic structures of STTD and organizes it into a matrix or a tensor format \citep{asif2016matrix}. By forcing the predicted STTD to be low-rank, a majority of system features can be reconstructed. Due to its clear definition and conceptual simplicity, it has attracted a large amount of interest in transportation research.
Notably, low-rankness is a widely studied structure. The low-rank learner uses the algebraic structures of STTD and organizes it into a matrix or a tensor format \citep{asif2016matrix}. Models based on it assist in reconstructing sparse data \citep{tan2013tensor,wang2018traffic,chen2019missing,nie2022truncated,lyu2024tucker}, detecting anomalies \citep{wang2021hankel,sofuoglu2022gloss}, revealing patterns \citep{avila2020data,wang2023anti}, forecasting future variables \citep{yang2021real}, and predicting unknown states \citep{zhang2020network,nie2023correlating,wang2023low,xing2023customized}. 
% There is particular interest in reconstructing data from sparse observations \citep{tan2013tensor,wang2018traffic,chen2019missing,nie2022truncated,lyu2024tucker}, detecting anomaly \citep{wang2021hankel,sofuoglu2022gloss}, discovering interpretable patterns \citep{avila2020data,wang2023anti}, forecasting future states \citep{yang2021real}, and estimating unobserved system states \citep{zhang2020network,nie2023correlating,wang2023low,xing2023customized}.

% While great progress has been made, these studies have focused either on methods that are not universally applicable to other data types and structures or have only demonstrated state-of-the-art results on specific problems that cannot be generalized to different tasks or scenarios. 
While great progress has been made by using low-rank models as learners of STTD, these methods have focused either on patterns that are applicable to specific data structures and dimensions or have demonstrated state-of-the-art results with source-dependent priors. This limits the potential for a unified representation and emphasizes the need for a generally applicable method to link various types of STTD learning.
% For instance, low-rank patterns may vary across different data scales. Additionally, the fitted matrix cannot generalize beyond the current resolutions.
% Most of them are incapable of generalizing without extensive retraining or optimizing from scratch. 
% Therefore, the development of a universally applicable method for general STTD analysis remains a formidable challenge and a significant gap in contemporary research.
% Nevertheless, a universal and efficient method for general STTD analysis is perhaps the most difficult and strongly lacking component of modern ITS \citep{zhang2011data}. 
% \textbf{[What properties are desirable and in which scenarios do we need a new tool]}
To address these limitations, we first carefully survey existing low-rank models and have identified the primary shortcomings that hinder their versatility.
(1) \textit{Absence of high-frequency components}: The rationale for low-rank models is to completely remove high-frequency components and treat them as noise. However, high-frequency components can provide informative details for the accurate reconstruction of spatial-temporal fields \citep{luo2024continuous}. Omitting this section of the signals will lead to a loss of information such as nonrecurrent events and phase transitions. 
(2) \textit{Resolution dependency}: Due to the matrix or tensor organization, these models are restricted to regular sampling intervals or discrete mesh grids. 
They often require a grid-based input with fixed spatiotemporal dimensions, restricting them from accommodating varying spatial resolutions or temporal lengths. 
Such a restriction makes them infeasible to work beyond the current discretization, which violates the continuous nature of STTD. More importantly, STTD is sometimes sensitive to the choice of discretization. For example, the definition of a space-time cell has an obvious impact on the reproduction of congestion wave \citep{he2017constructing}.
In addition, they are also memory intensive because the discrete representation requires the entire input of numerous grids.
(3) \textit{Explicit regularization}: To ensure favorable performance, some structural priors are usually explicitly imposed on standard low-rank models. For instance, a predefined small rank or a surrogate norm function is necessary to facilitate the optimization process \citep{goulart2017traffic,chen2020nonconvex}. 
Such a rank or norm selection procedure can be tricky,
% , and rigorous rank constraints may reduce the expressivity of the model. 
and the low-rank pattern modeling, fixed on one data source, may not generalize to different data sources. For instance, low-rankness identified in one data type, such as vehicle trajectories, may not be applicable to differently structured data, such as OD demand.
Moreover, some prior regularization terms are also needed to regulate the local consistency of STTD \citep{yu2016temporal,chen2022laplacian,nie2023correlating}.  
These hand-crafted a priori are eventually processed into heuristic penalty functions and a choice of several hyperparameters, which are scenario-based and less adaptable to different problems.
% \textcolor{gray}{(4) \textbf{Instance-specific learning}: Most of the low-rank models for STTD are optimized in a transductive way, which means that they are fitted to each STTD instance with a set of parameters separately. And they usually require different hyperparameters to function across different settings. When faced with a new instance, they entail re-optimization and re-select parameters, which poses great challenges for their generalizability. A model featured both generalizable \citep{nie2023imputeformer} and inductive \citep{wu2021inductive} capabilities is more preferable in practice.}

% \textbf{[How we achieve this goal and exploit recent advances in continuous data representation with neural function representation.]}

To unravel these difficulties, we attempt to learn the data-generating dynamics of STTD directly. Fortunately, deep learning approaches exist for a promising facet of this problem.
Recent studies have empirically demonstrated the surprising ability of deep neural networks to learn predictive representations of dynamical systems, such as physics-informed neural networks for traffic speed and density field estimation \citep{huang2020physics,shi2021physics,zhang2024physics,zheng2024recovering}. 
However, it seems impossible to derive authentic governing equations or explicit regularities of these complex traffic dynamics in all scenarios due to unpredictable factors in real-world conditions. This prompts us to develop data-driven \textit{implicit} techniques.
% Implicitly defined, continuous, differentiable data representation models parameterized by deep neural networks have recently emerged as a powerful paradigm, offering many benefits over conventional representations \citep{SIREN,FourierFeature,chen2021learning}.
Recently, there has been a rise in the prominence of implicitly defined, continuous, and expressive data representation models parameterized by deep neural networks. This novel paradigm uses neural networks to discern patterns from continuous input.
It has proven to be powerful in representing images, videos, and point cloud data, offering numerous advantages over traditional representations \citep{SIREN,FourierFeature,chen2021learning}. 
The implicit neural {representations} (INRs) function in a continuous function space and take domain coordinates as input, predicting the corresponding quantity at the queried coordinates. 
% INRs directly model the mapping from coordinates or derivatives in implicit manifolds, which is necessary to represent processes defined implicitly by some governing dynamics. 
INRs directly model the mapping from low-dimensional regimes, such as coordinates and derivatives, to high-frequency structures, which is necessary to represent processes defined implicitly by some governing dynamics. 
% INRs learn patterns in implicit manifolds and fit processes that generate target data with functional representation.
This differentiates them from low-rank models that depend on explicit dominating patterns, enhancing their expressivity in learning complex data details, and enabling them to fit processes that generate target data with functional representation. 
Consequently, the continuous representations eliminate the need for fixed data dimensions and can adjust to STTD of any scale or resolution, allowing us to model various STTD with a unified input.
In this work, we exploit the advances of INRs and tailor them to incorporate the characteristics of STTD, resulting in a novel method that serves as a general traffic data learner (see Fig. \ref{fig:intro}a).

% \textbf{[Apart from the representation power, what other important features are notable and beneficial for traffic data modeling? How we achieve this?]}

Specifically, we parameterize the implicit mapping from input domain to traffic states using coordinate-based multilayer perceptrons (MLPs) (see Fig. \ref{fig:intro}b). To learn complex details within the definition domain, we encode high-frequency components into the input of MLPs. To unravel the entangled relationships between spatial and temporal factors, we decompose the variability into separate processes through coordinate disentanglement. To model STTD exists in irregular space such as a sensor graph, we formulate a spectral embedding technique to learn non-Euclidean mappings.
% \textcolor{gray}{We then integrate it with a meta-learning strategy to decode instance-wise patterns, making it generalizable to different STTD instances.}
In addition to the versatility of the representation, we also theoretically show that it possesses several salient features that are crucial for STTD in practical applications. They include: (i) implicit low-rank regularization derived from the gradient descent over deep matrix factorization; (ii) inherent smoothness from the continuity of MLPs; and (iii) reduced computational complexity. 
% Our framework seamlessly integrates deep learning and matrix (tensor) factorization, with the aim of taking advantage of the representation and generalization power of deep learning and the prominent traits of low-dimensional methods.
As a result, the proposed method explicitly encodes high-frequency structures to learn complex details of STTD while at the same time implicitly learning low-rank and smooth priors from data to reconstruct the dominating modes.
Through the lens of generalized representations, 
% we can better predict the collective behaviors of STTD and understand the underlying spatiotemporal dynamics that drives this complex system. 
we can explore the possibility of developing a task-agnostic base learner that can predict the collective behaviors of STTD and address a wide variety of learning problems in reality (see Fig. \ref{fig:intro}c).

% As a proof-of-concept, we demonstrate the effectiveness of the approach in various well-designed benchmarking tasks, covering scales ranging from corridor, gird, to network. Furthermore, it is demonstrated how this methodology can be generalized to different input conditions, data domains, output resolutions, and network topology. 
Our proof-of-concept has shown promising results through extensive evaluations using real-world datasets. The proposed learner is versatile, working across different scales - from corridor-level to network-level applications. It can also be generalized to various input dimensions, data domains, output resolutions, and network topologies. Comprehensive model analyses are also provided to fully examine its working mechanism.
This study offers novel perspectives on STTD modeling and provides an extensive analysis of practical applications, contributing to state-of-the-art STTD learning methods.
% To our knowledge, this is the first modeling paradigm that integrates essential properties for universal STTD learning and is generalizable to a variety of real-world tasks.
% We hope this will lay the groundwork for developing foundational models for STTD.
To our knowledge, this is the first time that INRs have been effectively applied to STTD learning and have demonstrated promising results in a variety of real-world tasks. 
We anticipate this could form the basis for developing foundational models for STTD.
Our main contributions are summarized as follows:
\begin{enumerate}
    \item A new paradigm for generalized STTD learning is presented that parameterizes traffic variables as an implicit mapping of domain coordinates;
    \item A coordinate disentanglement strategy is proposed to dissect the complexity of learning on full domains and a graph spectral embedding technique is developed to model irregular data on topological spaces;
    \item Salient analytical properties of our model that can serve as regularization for STTD are provided, including high-frequency 
    encoding, implicit low-rankness, and inherent smoothness;
    % \item We demonstrate its generality by a variety of real-world traffic data learning problem under different scales, including highway corridor, urban grid, and road network.
    \item Our comprehensive experiments provide the first practice that demonstrates the effectiveness of INRs on extensive real-world STTD learning problems, covering data from micro to macro scales.
\end{enumerate}

The remainder of this paper is organized as follows. Section \ref{Literature review} reviews the existing work on STTD according to different tasks. 
% Section \ref{Notations and Problem Definitions} provides preliminaries about low-rank models and INRs. 
Section \ref{methodology} introduces the STTD learning problem, elaborates on our method, and performs a theoretical analysis. Section \ref{Sec: experiments} evaluates it on several real-world scenarios. Section \ref{Sec: Algorithmic analysis} discusses the modular designs in detail. Section \ref{conclusions} concludes this work and provides future directions.

\section{Related Works}\label{Literature review}
Existing approaches for STTD span multiple research areas and are briefly discussed in this section to recapitulate the motivation of our paper. In particular, we detail low-rank methods for traffic data reconstruction as a representative. We also briefly review the recent advances of INRs in the machine learning community.

\subsection{Spatiotemporal traffic data modeling}
Since STTD can be generated on various scales of transportation systems, it has a wide spectrum of applications. 
Current studies usually combine data-driven models with STTD to achieve a particular modeling purpose. Representative tasks include: {extracting high-granularity vehicle trajectories for highway simulation \citep{shi2021video}},
data reconstruction on road segment \citep{wang2018traffic,bae2018missing}, graph-based sensor data imputation \citep{deng2021graph}, deep learning-based imputation \citep{shi2021physics,liang2022spatial,liang2022memory,nie2024imputeformer}, prediction of unmeasured traffic on urban road segment \citep{zhang2020network}, highway speed extrapolation (kriging) \citep{wu2021inductive,nie2023correlating}, highway traffic volume estimation \citep{nie2023towards}, transfer-based volume estimation \citep{zhang2022full}, 
anomaly detection \citep{qin2019probdetect,wang2021hankel}, data denoising \citep{zheng2024recovering}, learning relational structures \citep{lei2022bayesian}, dynamics modeling and analysis \citep{avila2020data,lehmberg2021modeling,wang2023anti}, forward and backward dynamics modeling \citep{thodi2024fourier}, real-time traffic forecasting \citep{yang2021real}, origin-destination flow forecasting \citep{zhang2021short,cheng2022real}, individual trajectory reconstruction \citep{chen2024macro}, reconstruction of recurrent spatial-temporal traffic states at intersections \citep{wang2024traffic}, corridor speed field estimation \citep{wang2023low}, link travel time estimation \citep{li2023filtering,fu2023optimization}, estimating network-wide speed matrix \citep{liu2019spatial,yu2020urban}, 
{probabilistic imputation and prediction \citep{xu2023agnp}},
macroscopic network state estimation \citep{saeedmanesh2021extended}, empirical analysis of large-scale multimodal network \citep{fu2020empirical}.
{In general, existing studies on this topic either focus on data-centric problems related to the properties and quality of the data itself, or address a domain-specific problem in real-world traffic engineering.}

\subsection{Low-rank models for spatiotemporal traffic data}
Among the STTD modeling tasks, sensor-based traffic data imputation has attracted particular interest in recent years.
In this problem, low-rank models are preferred for their simplicity and interpretability. The multivariate sensor time series is first organized into a spatiotemporal matrix or tensor, then low-rank models such as matrix factorization, nuclear norm minimization, tensor factorization, and tensor completion are adopted to fill in the missing data. A pioneering work \citep{tan2013tensor} introduced a Tucker tensor model to model multidimensional traffic data. \cite{asif2016matrix} studied the effectiveness of different matrix- and tensor-based methods to estimate missing data in traffic systems. After these works, numerous studies have emerged to develop more advanced low-rank surrogates \citep{goulart2017traffic,zhang2019missing,chen2019missing,chen2019bayesian,chen2020nonconvex,chen2021scalable,nie2022truncated,xing2023customized}, propose new regularization schemes \citep{wang2018traffic,chen2022laplacian,nie2023correlating}, and integrate them with time series processing frameworks \citep{chen2021Autoregressive,wang2023low,lyu2024tucker,mei2024high,nie2024channel}. {Although low-rank models are efficient for dealing with high-dimensional STTD, existing low-rank methods can only fitted to a single data source with fixed resolution due to the discrete organization. In addition, the elaborated regularization scheme is task-oriented, thereby sacrificing its flexibility in different problems.}

\subsection{Implicit data and function representation}
INRs are an emerging paradigm for representing data and functions in machine learning community, such as images, videos, 3D scenes, point cloud, and audio data \citep{SIREN,FourierFeature,mildenhall2021nerf,chen2021learning,dupont2021coin}. In particular, \cite{SIREN} proposed a simple-yet-effective periodic activation to help INRs learn high-frequency features. \cite{FourierFeature} provided a in depth discussion on the role of Fourier features in deep neural networks using the theory of the neural tangent kernel. \cite{mildenhall2021nerf} developed a neural radiance fields model with positional encoding to synthesize novel views of complex scenes.
A recent work \citep{luo2023low} extended INRs to tensor structures and proposed a tensor functional representation for visual and point-cloud data.
INRs have also been extended to model time series, termed time-index models \citep{fons2022hypertime,naour2023time,woo2023learning}. 
{Existing efforts in advancing INRs mainly concentrate on multi-modal data in machine learning tasks.}
However, the application of INRs in spatiotemporal data, especially for traffic data, is still lacking. {As far as we know, there are no existing INR-based methods that can be directly applied to the problem studied in this paper.}

\section{ST-INR for Generalized Traffic Dynamics Learning}\label{methodology}
To begin with, this section first introduces the key concepts and notation used throughout the paper. 
% We then elucidate some background about the STTD learning problem and a common modeling framework.
Regarding the notation, we follow the terminology of \citep{kolda2009tensor}. Specifically, matrices are denoted by boldface capital letters e.g., $\mathbf{A}\in\mathbb{R}^{M\times N}$, vectors are represented by boldface lowercase letters, e.g., $\mathbf{a}\in\mathbb{R}^{M}$ and scalars are lowercase letters, e.g., $a$. Without ambiguity, functions are abbreviated as $\Phi(\cdot)$. Without special remarks, we use calligraphic letters to denote the vector space, e.g., $\mathcal{X}\subseteq\mathbb{R}$. In particular, loss function is signified as $\mathcal{L}(\cdot)$.
Next, we will elaborate the proposed spatiotemporal implicit neural representation (ST-INR) model. We will begin by introducing the generalized modeling framework from a broad perspective. Then we detail each modular design to account for the characterization of STTD. Finally, we provide theoretical discussions to indicate some important properties that are crucial for real-world applications. The overall architecture of the proposed ST-INR is shown in Fig. \ref{fig:overall_arch} (three-dimensional data is used as an example).
% We take the form of a matrix to present the model as a demonstration. The high-dimensional tensor extensions are formulated in \ref{Appendix: tensor}.

\begin{figure}[!htb]
\centering
\includegraphics[width=0.99\textwidth]{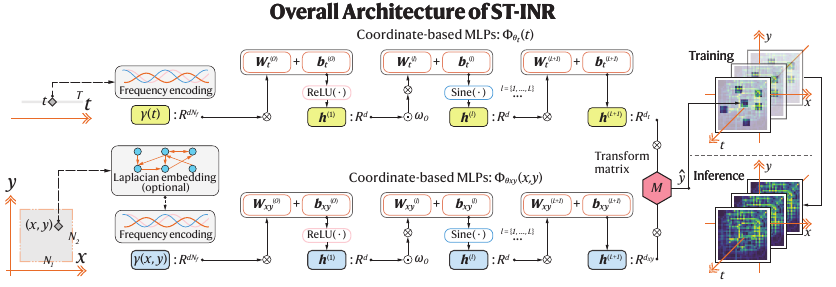}
\caption{\textbf{Overall architecture of the proposed ST-INR model (three-dimensional case).}}
\label{fig:overall_arch}
\end{figure}

\subsection{Generalized traffic data leaner through implicit neural representation}
% implicit representation models vs. explicit representation models. 
STTD usually exhibits explicit patterns, e.g., the temporal dynamics of traffic time series can show continuity, periodicity, and nonstationarity. However, the collected real-world traffic data often shows complex regularities due to undesirable observational conditions. It often suffers from sparse and noisy data, contains anomalies, and can be partially measured. To understand and exploit STTD to characterize traffic, practitioners resort to data-driven models to learn, reconstruct, denoise, analyze, and forecast partially observed STTD.

Generally, STTD can be generated from a complex process involving their locations in space and time domain. For instance, vehicle trajectories are defined by $x-t$ coordinates. Origin-destination flows include both the $x-y$ locations and the time point $t$. Sensor-based time series is associated with their relational position and time interval.
On top of this understanding, the traffic dynamics are assumed to obey a system of equations:
\begin{equation}\label{eq:inr}
    \mathcal{F}(\mathbf{x},\Phi,\nabla_{\mathbf{x}}\Phi,\nabla^2_{\mathbf{x}}\Phi,\dots)=0,
\end{equation}
where $\mathbf{x}$ is the spatial (or spatial-temporal) coordinates of the input domain, and the function $\Phi:\mathbf{x}\mapsto\Phi(\mathbf{x})$ maps the coordinates to some quantities of interest, such as the traffic states at some spatiotemporal points.
Since $\Phi$ is defined by the constraint $\mathcal{F}$ (which possibly contains the derivatives), it is \textit{implicitly} modeled and usually parameterized by a deep neural network $\Phi_{\theta}$ \citep{SIREN}. As a result, $\Phi_{\theta}$ defines the implicit neural representation (INR) to approximate the explicit solution to Eq. \eqref{eq:inr}. Given the true quantity $f(\mathbf{x})$ at any location, INR can serve as a data-driven \textbf{learner} that can model the regularities of various traffic data by minimizing the continuous loss function in the entire definition space:
\begin{equation}
    \min_{\theta} \mathcal{L} = \int_{\mathcal{D}}\Vert\mathcal{F}(f(\mathbf{x}),\Phi_{\theta},\nabla_{\mathbf{x}}\Phi_{\theta},\nabla^2_{\mathbf{x}}\Phi_{\theta},\dots)\Vert d\mathbf{x},
\label{eq:inr_intro}
\end{equation}
where $\mathcal{D}$ is the definition domain of $\mathbf{x}$. Due to its differentiable property, Eq. \eqref{eq:inr_intro} can be optimized by the gradient descent method.
% To obtain a data-driven \textbf{learner} that can model the regularities of various traffic data, STTD modeling problems can be formulated as fitting a general learning model parameterized by $\theta\in\Theta$:
With the spatiotemporal context, the learning problems of STTD on observed training data can be instantiated as a special form of Eq. \eqref{eq:inr_intro} parameterized by $\theta\in\Theta$:
\begin{equation}\label{eq:general_model}
\min_{\theta}\mathcal{L}\left(\Omega(\Phi_{\theta}(\mathbf{x},\mathbf{t})),\mathbf{y}\right)+\lambda\mathcal{R}(\Phi_{\theta}(\mathbf{x},\mathbf{t})),
\end{equation}
where $\Phi_{\theta}:\mathcal{X}\subseteq\mathbb{R}^{N_x}\times\mathcal{T}\subseteq\mathbb{R}^{N_t}\mapsto\mathbb{R}^{D_{\text{out}}}$ is the implicit mapping function, i.e., the spatiotemporal learner, $\mathbf{x}\in\mathcal{X},\mathbf{t}\in\mathcal{T}$ are the space-time coordinates,
$\mathbf{y}\in\mathbb{R}^{D_{\text{out}}}$ is the observation that can be partially measured or noisy,
$\Omega$ is a sampling operator, $\mathcal{R}(\cdot)$ is a regularization defined on the continuous space to impose some priors over $\Phi_{\theta}$, and $\lambda$ is the weight hyperparameter.
Importantly, we treat STTD as a vector field and its regularities can be controlled by some underlying partial differential equations. Therefore, its state at an arbitrary location $x$ and time $t$ is subject to both spatial and temporal coordinates.
For discrete problems, the definition domains can be discrete sets, e.g., $\mathcal{X}=\{0,1,2,\dots\}$.
For continuous problems, a continuous domain is considered, e.g., $\mathcal{T}\in\mathbb{R}^+$. {We illustrate the learning objective of Eq. \eqref{eq:general_model} in Fig. \ref{fig:loss_func}.}

\begin{figure}[!htb]
\centering
\includegraphics[width=0.6\textwidth]{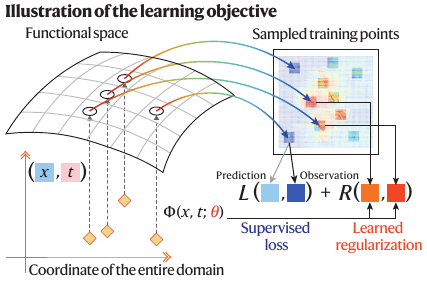}
\caption{{\textbf{Illustration of the learning objective of ST-INR.} The total objective includes a supervised loss over all observed points, and a learnable regularization derived from the property of deep neural networks.}}
\label{fig:loss_func}
\end{figure}

In conventional low-rank models, Eq. \eqref{eq:general_model} is instantiated by two components. The reconstruction loss $\mathcal{L}$ is implemented by the approximation errors of matrix factorization or a surrogate nuclear norm minimization. The regularization term $\mathcal{R}$ is formulated as an explicit penalty function such as total variation, graph Laplacian, and kernel function. In this work, we model them in a unified way by imposing implicit structures of regularization learned from data, which makes our method model-agnostic and adaptive for different patterns.

Drawing inspiration from INRs, we let multi-layer perceptrons (MLPs) be the parameterization $\theta$. With the universal approximation theorem, MLPs with infinite depth can theoretically fit arbitrary continuous functions \citep{hornik1991approximation}. Concretely, taking the two-dimensional STTD as an example, the function representation is expressed as a continuous mapping from $(x,t)$ coordinates to traffic state variables:
% \begin{equation}
%     \Phi_{\theta}(x,t):\mathcal{X}\times\mathcal{T}\mapsto\mathcal{Y}
% \end{equation}
\begin{equation}
\begin{aligned}
    &\Phi_{\theta}(x,t): (x,t)\mapsto\mathcal{F}_{\theta}(x,t)\in\mathbb{R}^{D_{\text{out}}},\\
    &\mathcal{F}_{\theta}(x,t)=\mathbf{W}^{(L+1)}(\phi^{(L)}\circ\phi^{(L-1)}\circ\cdots\circ\phi^{(0)})([x,t])+\mathbf{b}^{(L+1)}, \\
    &\mathbf{x}^{(\ell+1)}=\phi^{(\ell)}(\mathbf{x}^{(\ell)})=\sigma(\mathbf{W}^{(\ell)}\mathbf{x}^{(\ell)}+\mathbf{b}^{(\ell)}), \\
\end{aligned}
\end{equation}
% where $\mathcal{X}\subseteq\mathbb{R}^{N}$ is the spatial domain and $\mathcal{T}\subseteq\mathbb{R}^{+}$ is the temporal domain, and $\mathcal{Y}\subseteq\mathbb{R}$ is the output domain. 
where $\mathcal{F}_{\theta}$ is the coordinate-based MLPs \citep{SIREN}, $\mathbf{x}^{(\ell)}$ is the layerwise input with $\mathbf{x}^{(0)}=[x,t]$, $\sigma$ is the elementwise activation function, and $\theta=\{\mathbf{W}^{(\ell)},\mathbf{b}^{(\ell)}|\ell=0,\dots,L\}\cup\{\mathbf{W}^{(L+1)},\mathbf{b}^{(L+1)}\}$ are the model parameters need to be inferred.
% and $\phi^{(\ell)}$ is defined in Eq. \eqref{eq:mlp}.
As an example, we can represent the speed field on regular mesh grid $\mathbf{X}\in\mathbb{R}^{n\times t}$ with an INR $\Phi_{\theta}(x,t):\mathbb{R}\times\mathbb{R}\mapsto\mathbb{R}$ that satisfies $\Phi_{\theta}(i,j)=\mathbf{X}[i,j],i\in\{0,1,\dots,n\},j\in\{0,1,\dots,t\}$. After learning the parameters, INR can predict the speed value at arbitrary coordinates $(x,t)\in\mathbb{R}\times\mathbb{R}$, even beyond the original discrete grid $[i,j]$.

\subsection{Encoding high-frequency components in functional representation}
Low-rank models such as MF focus mainly on the utilization of low-frequency components. They assume that the reconstruction is dominated by a small number of modes. However, we argue that although STTD shows dominating patterns, high-frequency components can encode complex details and structures about STTD. Especially, high-frequency parts can relate with significant local patterns \citep{sen2019think}. Nevertheless, learning the high-frequency tails of the spectrum is a longstanding challenge for machine learning methods. The famous \textit{frequency principle} revealed by Fourier analysis indicates that deep neural networks tend to fit low-frequency components during model training \citep{Frequency-principle}. This is understood as the ``spectral bias'' of deep neural networks \citep{rahaman2019spectral}.
To alleviate this spectral bias, we adopt two advanced techniques to learn high-frequency components. We first denote the spatial-temporal input coordinate as $\mathbf{v}=(x,t)\subseteq\mathbb{R}\times\mathbb{R}^+$.

The first strategy is to equip $\mathcal{F}_{\theta}$ with periodic activation functions \citep{SIREN}:
\begin{equation}\label{eq:siren}
\phi_s^{(\ell)}(\mathbf{v}^{(\ell)})=\sin(\omega_0\cdot\mathbf{W}^{(\ell)}\mathbf{v}^{(\ell)}+\mathbf{b}^{(\ell)}),~\ell=\{1,\dots,L\},
\end{equation}
where $\omega_0$ is a frequency factor to ensure the sine function spans multiple periods over $[-1,1]$, and the weights are initialized as $\mathbf{w}^{(\ell)}_i\in\mathbb{R}^D\sim\mathcal{U}(-\sqrt{6/d},\sqrt{6/d})$ to have a standard deviation of 1 \citep{SIREN}. The sine activation introduces periodicity to the hidden states and its gradient is also a sine function everywhere, which makes the neural networks easily learn complicated patterns. 

The second solution is to use concatenated random Fourier features (CRF) in the input layer \citep{FourierFeature,benbarka2022seeing}. We can select problem-specific Fourier features and incorporate them into the input of $\mathcal{F}_{\theta}$. However, obtaining a predefined frequency can be challenging due to the complex patterns of STTD.
Therefore, unlike the original method in \citep{FourierFeature,benbarka2022seeing}, we employ a series of CRF with different scales in the input layer to explicitly inject the high-frequency components:
\begin{equation}\label{eq:crf}
    \gamma(\mathbf{v})=[\sin(2\pi\mathbf{B}_1\mathbf{v}),\cos(2\pi\mathbf{B}_1\mathbf{v}),\dots,\sin(2\pi\mathbf{B}_{N_f}\mathbf{v}),\cos(2\pi\mathbf{B}_{N_f}\mathbf{v})]^{\mathsf{T}}\in\mathbb{R}^{d{N_f}},
\end{equation}
where the Fourier basis frequency $\mathbf{B}_k\in\mathbb{R}^{d/2\times c_{\text{in}}}$ are sampled from Gaussian distribution $\mathcal{N}(0,\sigma_{k}^2)$. By setting a large number of frequency features $N_f$ and a series of scale parameters $\{\sigma^2_k\}_{k=1}^{N_f}$, we can sample a variety of frequency patterns in the input domain, which is crucial for the reconstruction of complex patterns (see Section \ref{sec:frequency_analysis}). In addition, having diverse frequency features enables our model to be adaptive to different input data, simplifying the process of selecting a specific frequency values.
The rationale of using Fourier feature mapping can be explained by the theory of neural tangent kernel (NTK). \cite{FourierFeature} show that CRF can transform the effective NTK into a stationary kernel with a tunable bandwidth, enabling faster convergence for high-frequency components.

The overall workflow of a single frequency-enhanced $\mathcal{F}_{\theta}:\widehat{\mathbf{y}}=\Phi_{\theta}(\mathbf{v})$ can be formulated as follows:
\begin{equation}\label{eq:freqmlp}
\begin{aligned}
    &\mathbf{h}^{(1)} = \texttt{ReLU}(\mathbf{W}^{(0)}\gamma(\mathbf{v})+\mathbf{b}^{(0)}), \\
    &\mathbf{h}^{(\ell+1)} =\phi_s^{(\ell)}(\mathbf{h}^{(\ell)})= \sin(\omega_0\cdot\mathbf{W}^{(\ell)}\mathbf{h}^{(\ell)}+\mathbf{b}^{(\ell)}),~\ell=\{1,\dots,L\}, \\
    &\widehat{\mathbf{y}} = \mathbf{W}^{(L+1)}\mathbf{h}^{(L+1)}+\mathbf{b}^{(L+1)},
\end{aligned}
\end{equation}
where $\mathbf{W}^{(\ell)}\in\mathbb{R}^{d_{(\ell+1)}\times d_{(\ell)}},\mathbf{b}^{(\ell)}\in\mathbb{R}^{d_{(\ell+1)}}$ are layerwise parameters, and $\widehat{\mathbf{y}}\in\mathbb{R}^{d_{\text{out}}}$ is the predicted value. The combination of these two strategies achieves high-frequency low-dimensional regression, empowering the coordinate-based MLPs to learn complex details with high resolution.

\subsection{Factorizing spatial-temporal variability through coordinate disentanglement}\label{sec:factorization}
Using a single $\Phi_{\theta}$ to model entangled spatial-temporal interactions within a continuous domain can be challenging. Therefore, we decompose the spatiotemporal process into separate variables in two dimensions using coordinate disentanglement.
Specifically, we factorize the definition domain of $\Phi_{\theta}$ into two axes:
\begin{equation}\label{eq:factorize}
    \begin{aligned}
         % &\Phi_{\theta}(\mathbf{v})=\Phi_{\theta_x}(x)\Phi_{\theta_t}(t)^{\mathsf{T}}, \\
          &\Phi_{\theta}(\mathbf{v}) = \Phi_{\theta_x}(x)\mathbf{M}_{xt}\Phi_{\theta_t}(t)^{\mathsf{T}}, \\
         &\Phi_{\theta_x}:\mathcal{X}\mapsto\mathbb{R}, ~x\mapsto\Phi_{\theta_x}(x)\in\mathbb{R}^{d_x}, \\
          &\Phi_{\theta_t}:\mathcal{T}\mapsto\mathbb{R}, ~t\mapsto\Phi_{\theta_t}(t)\in\mathbb{R}^{d_t}, \\
    \end{aligned}
\end{equation}
where $\Phi_{\theta_x}$ and $\Phi_{\theta_t}$ are defined by Eq. \eqref{eq:freqmlp}.
To further align the two components, which can have different dimensions and frequency patterns, we adopt a middle transform matrix $\mathbf{M}_{xt}\in\mathbb{R}^{d_x\times d_t}$ to model their interactions in the hidden manifold.
This formulation makes Eq. \eqref{eq:factorize} a generalized and continuous matrix factorization model. When the spatial-temporal coordinates ${x}$ and ${t}$ are defined on a discrete grid, a sampled matrix $\mathbf{Y}$ can be given as $\{\mathbf{Y}[i,j]=\Phi_{\theta_x}(x^{(i)})\Phi_{\theta_t}(t^{(j)})^{\mathsf{T}},\forall i,j\in\mathbb{N}_+\}$, which intrinsically constructs a discrete matrix factorization model. On the contrary, it can process data or functions that exist beyond the regular mesh grid of matrices, capturing complex spatiotemporal patterns using the expressive power of deep learning.

% Finally, to further align the two components, which can have different dimensions and frequency patterns, we adopt a middle transform matrix $\mathbf{M}_{xt}\in\mathbb{R}^{d_x\times d_t}$ to model their interactions in the hidden manifold, which yields:
% \begin{equation}\label{eq:x_t_factorize}
%     \Phi_{\theta}(\mathbf{v}) = \Phi_{\theta_x}(x)\mathbf{M}_{xt}\Phi_{\theta_t}(t)^{\mathsf{T}}.
% \end{equation}

In practice, we can set $d_x=d_t=d_{xt}$ as the hidden dimension. $\mathbf{M}_{xt}$ is learned at the same time as other parameters in $ \Phi_{\theta_x}$ and $\Phi_{\theta_t}$ and we find that initializing it as an identity matrix can result in a good nontrivial solution. A significant difference between our model and existing low-rank models is that we adopt a large factorization dimension, i.e., the hidden dimension $d_{xt}$, to ensure the learning of a wider singular spectrum while at the same time encouraging a low-rank solution through the \textbf{implicit low-rank regularization} brought by gradient descent \citep{arora2019implicit}.
Important discussions about the two hidden dimensions are provided in Section \ref{sec:implicit_low_rank}.
It is also noteworthy that in this factorized formulation, the temporal and spatial components can adopt different frequency features, thereby learning different implicit mappings to encode complex data patterns. This coordinate disentanglement strategy also brings efficiency benefits (see Section \ref{sec:complexity}).

\subsection{Learning traffic data in arbitrary domains}\label{Sec: graph-embedding}
The mapping $\Phi_{\theta}$ enables us to parameterize the function spaces over traffic data. With this powerful and flexible tool, we seek to learn a wide range of traffic data, such as individual trajectory, sensor-based time series, dynamic origin-destination flow, and network traffic state, as a continuous function $\Phi_{\theta}$. However, in reality, traffic data can be sampled and stored in an irregular way. For example, a sensor network can be abstracted as a weighted directed graph where vertices denote detectors of the network, and edges indicate the connectivity (reachability) between sensors. In the Euclidean case, we already know the regularity of the data domain, e.g., $\mathcal{X},\mathcal{T}\subseteq\mathbb{R}$, and INRs can be readily trained on a discrete sampling of a signal over a regular lattice. 
However, when there are traffic data on an arbitrary topological space $\mathcal{H}$, instead of a regular lattice such as the $x-t$ coordinate, we cannot represent $\mathcal{H}$ in a canonical coordinate system. 
To ease the difficulty, we assume that we can obtain a discrete graph realization $\mathcal{G}$ sampled from the unknown topological space of the continuous signal. We then approach the non-Euclidean data modeling problem by projecting it to the Laplacian eigenvector space.

% Then the implicit function $\Phi$ is trained to map the spectral embedding of each node (coordinated in the Laplacian eigenspace) to the graph signal values. 

Specifically, in this generalized setting, we consider a graph signal $f:\mathcal{H}\mapsto\mathcal{Y}$ with its discrete graph realization $\mathcal{G}(V,E,\mathbf{A})$. The basic elements are node set $V=\{v_i|i=1,\dots,n\}$, edge set $E\subseteq V\times V$ and weighted adjacency matrix $\mathbf{A}\in\mathbb{R}^{n\times n}$ (undirected). 
We resort to the graph spectral embedding technique \citep{grattarola2022generalised}, bypassing the need to know the continuous space underlying the traffic graph signals. 
Since the graph describes the relations between graph signals, we can obtain a meaningful coordinate system by manipulating the spectral embedding of the graph Laplacian \citep{grattarola2022generalised}. Given the adjacency matrix $\mathbf{A}$, the symmetric normalized Laplacian as well as its eigendecomposition are given as:
\begin{equation}
\begin{aligned}
    \mathbf{L}_{\text{s}}=\mathbf{I}-\mathbf{D}^{-\frac{1}{2}}\mathbf{A}\mathbf{D}^{-\frac{1}{2}}=\mathbf{U}^{\mathsf{T}}\operatorname{diag}(\lambda_1,\dots,\lambda_n)\mathbf{U}, 
\end{aligned}
\label{Laplacian}
\end{equation}
where $\mathbf{D}=\operatorname{diag}(\sum_{i'}a_{ii'})$ is the degree matrix, $\mathbf{U}$ is the matrix whose column is the eigenvectors and the eigenvalues are ordered as $\lambda_1\leq\lambda_2\leq\dots,\lambda_n$. Correspondingly, the associated eigenvectors form an orthonormal basis $\{\mathbf{u}_j\in\mathbb{R}^n,j=1,\dots,n\}$, which is treated as the coordinate in the Laplacian eigenspace for each node $v_j$:
\begin{equation}\label{eq:graph_embedding}
    \mathbf{e}_j=[u_{1,j},u_{2,j},\dots,u_{n,j}]^{\mathsf{T}}\in\mathbb{R}^{n}.
\end{equation}

In practice, we truncate the first $k$ values as the final embedding $\widetilde{\mathbf{e}}_j[:k]\in\mathbb{R}^{k}$ to reduce complexity. 
Similar to the routine of regular INRs, a neural network $\Phi_{\theta}:\widetilde{\mathbf{e}}_j\mapsto f(v_j)$ is trained to map the spectral embedding of each node (coordinate in the Laplacian eigenspace) to the graph signal value. 
After model training, we can calculate the corresponding spectral embedding to infer the graph signal values for any node, as long as we can sample from the discrete graphs. Note that eigenvectors of the graph Laplacian are a discrete approximation of the continuous eigenfunctions of the Laplace operator on $\mathcal{H}$ \citep{GSP}, which forms a natural extension from the input in the Euclidean domain to the non-Euclidean domain. Importantly, if the signals are defined on spatiotemporal graphs, such as sensor-based time series, it can also be transformed into a vertex-time factorization formulation, just as Eq. \eqref{eq:factorize}.

\begin{algorithm}[!t]
\caption{Spatiotemporal implicit neural representations of traffic data}\label{algo:stinr}
\KwIn{Network parameter $\theta$, partially measured dataset $\mathcal{D}=\{(\mathbf{v}_i,\mathbf{y}_i)\}_{i=1}^M$.}
\KwOut{Trained ST-INR model $\Phi_{\theta}$ and the predictions at queried coordinates $\{\widehat{\mathbf{y}}_i\}_{i=1}^{M^{\star}}$.}
\tcp{Model training stage}
Initialize $\Phi_{\theta_r}, \forall r\in\{1\dots n_{c_{\text{in}}}\}$ and the transform tensor $\mathcal{M}$; \\
Sample the basis frequency matrices $\mathbf{B}_k\sim\mathcal{N}(0,\sigma_k^2),~\forall k\in\{1,\dots,N_f\}$; \\
\While{not convergence}{
Sample a batch $\mathcal{B}$ of data paris $\{(\mathbf{v}_j,\mathbf{y}_j)\}_{j\in\mathcal{B},\mathcal{B}\subseteq \mathcal{D}}$; \\
% Set latent code to zeros $\phi^{(j)}\leftarrow 0,\forall j\in\mathcal{B}$ ;\\
% \tcp{Inner loop for latent modulations}
% \For{$s=1:N_{\text{inner}}$ and $j\in\mathcal{B}$}{$\phi^{(j)}\leftarrow\phi^{(j)}-\alpha\nabla_{\phi}\mathcal{L}(\Phi_{\theta,h_{\omega}{(\phi)}},\{(\mathbf{v}_i^{(j)},\mathbf{y}_i^{(j)})\}_{i\in M})|_{\phi=\phi^{(j)}}$;\\
% }
\tcp{Forward process}
\For{$i\in \mathcal{B}$}{
$\Phi_{\theta}(\mathbf{v}_i) \leftarrow \mathcal{M}\times_1\Phi_{\theta_1}(v_1)\times_2\cdots\times_{c_{\text{in}}}\Phi_{\theta_{c_{\text{in}}}}(v_{c_{\text{in}}})$; \\}
\tcp{Gradient descent for updating parameters} 
$\theta\leftarrow \theta - \eta \nabla_{\theta}\frac{1}{|\mathcal{B}|}\sum_{i\in\mathcal{B}}\Vert\mathbf{y}_i-\Phi_{\theta}(\mathbf{v}_i) \Vert_2^2$;\\
}
\tcp{Model inference stage}
Evaluate $\Phi_{\theta}$ at given queried coordinates within the definition domain $\{\mathbf{v}_i^{\star}\}_{i\in M^{\star}}$;\\
% , set $\phi^{\star}\leftarrow 0$ ;\\
% \For{$s=1:N_{\text{inner}}$}{$\phi^{\star}\leftarrow\phi^{\star}-\alpha\nabla_{\phi}\mathcal{L}(\Phi_{\theta,h_{\omega}{(\phi^{\star})}},\{(\mathbf{v}_i^{\star},\mathbf{y}_i^{\star})\}_{i\in M})|_{\phi=\phi^{\star}}$;\\
% }
\For{$i\in M^{\star}$}{
$\widehat{\mathbf{y}}_i\leftarrow\Phi_{\theta}(\widehat{\mathbf{v}}_i^{\star}) \leftarrow \mathcal{M}\times_1\Phi_{\theta_1}(v_1^{\star})\times_2\cdots\times_{c_{\text{in}}}\Phi_{\theta_{c_{\text{in}}}}(v_{c_{\text{in}}}^{\star})$; \\}
% Evaluate $\Phi_{\theta,h_{\omega}}(\mathbf{v}^{\star})$ for any $\mathbf{v}^{\star}\in\mathcal{X}^{\star}\times\mathcal{T}^{\star}$.
\end{algorithm}

% \subsection{Instance-specific representation with gradient optimization}
\subsection{High-dimensional extensions and gradient-based optimization}
\label{Appendix: tensor}
Eq. \eqref{eq:factorize} describes the two-dimensional case that can model an arbitrarily spatial-temporal traffic data matrix. For higher-dimensional structures, such as the time-varying origin-destination flows and grid-based traffic states, we can represent them in an expanded formulation.
Taking inspiration from \citep{luo2023low}, we can organize the multidimensional data into a generalized Tucker tensor \citep{kolda2009tensor} format. 
Different from the conventional Tucker model defined on discrete grids, ST-INR are fully defined on continuous input domains and the factor matrices are parameterized by INRs.
Specifically, given a single input-output data pair $(\mathbf{v},\mathbf{y})$ where $\mathbf{v}\in\mathbb{R}^{c_{\text{in}}}$ is the $c_{\text{in}}$ dimensional input coordinate and $\mathbf{y}\in\mathbb{R}^{c_{\text{out}}}$ is the corresponding true data value. The ST-INR model for high-dimensional data can be formulated as:
\begin{equation}\label{eq:full_model}
\begin{aligned}
    % &\min_{\Theta}\mathcal{L}(\Theta;\mathbf{x})=\frac{1}{M}\sum_{i=1}^M\Vert\mathbf{y}_i-\Phi_{\Theta}(\mathbf{v}_i) \Vert_2^2, \\
    &\Phi_{\theta}(\mathbf{v}) = \mathcal{M}\times_1\Phi_{\theta_1}(v_1)\times_2\cdots\times_{c_{\text{in}}}\Phi_{\theta_{c_{\text{in}}}}(v_{c_{\text{in}}}), \forall (v_1,v_2,\dots,v_{c_{\text{in}}})\in\mathbf{v},\\
    &\Phi_{\theta_i}: v_i\mapsto \Phi_{\theta_i}(v_i)\in\mathbb{R}^{n_i},~\forall i\in\{1,\dots,c_{\text{in}}\},
\end{aligned}
\end{equation}
where $\mathcal{M}\in\mathbb{R}^{n_1\times \dots n_{c_{\text{in}}}}$ is the core tensor, and $\theta=\{\theta_1,\theta_2,\dots,\theta_{c_{\text{in}}}\}\cup\mathcal{M}$ are the model parameters. Similarly, each $\Phi_{\theta_i}$ can be defined on a topological space with different base frequencies. The factorized size $n_i$ can also be full-dimensional, as described in Section \ref{sec:factorization}.

Given a data instance, for example, a bounded spatiotemporal speed contour, we can sample a set containing $M$ collocation pairs $\mathcal{D}=\{(\mathbf{v}_i,\mathbf{y}_i)\}_{i=1}^M$ where $\mathbf{v}_i\in\mathbb{R}^{c_{\text{in}}}$ is the input coordinate and $\mathbf{y}_i\in\mathbb{R}^{c_{\text{out}}}$ is the true data value. For example, if the speed contour has $N$ space grids and $T$ time points, then we have $M=NT$, $c_{\text{in}}=2$, and $c_{\text{out}}=1$.
Given training data $\mathbf{x}$ and the model $\Phi_{\theta}$, we can evaluate $\Phi_{\theta}$ in Eq. \eqref{eq:full_model} at all collocation points. Then Eq. \eqref{eq:general_model} can be instantiated in the following form:
\begin{equation}\label{eq:single_inr}
    \min_{\theta}\mathcal{L}(\theta;\mathcal{D})=\frac{1}{M}\sum_{i=1}^M\Vert\mathbf{y}_i-\Phi_{\theta}(\mathbf{v}_i) \Vert_2^2.
\end{equation}

Note that the regularization $\mathcal{R}$ in Eq. \eqref{eq:general_model} is learned implicitly from the data instance, rather than an explicit penalty term. We will highlight the benefits of implicit regularization in Section \ref{Sec: theoretical_analysis}.

The above equation is fully differentiable, we can optimize it using gradient-based optimizer, i.e.,
\begin{equation}
    % \theta\leftarrow \theta - \eta \nabla_{\theta}\frac{1}{|\mathcal{B}|}\sum_{j\in\mathcal{B}}\mathcal{L}(\theta,\{(\mathbf{v}_j,\mathbf{y}_j)\}_{j\in \mathcal{B}}).
    \theta_{k+1}\leftarrow \theta_k - \eta \nabla_{\theta}\mathcal{L}(\theta,\{(\mathbf{v}_j,\mathbf{y}_j)\}_{j\in\mathcal{B},\mathcal{B}\subseteq \mathcal{D}})|_{\theta=\theta_k},
\end{equation}
where $k$ is the gradient step, $\eta$ is the learning rate, and $\mathcal{B}$ is the data batch. More advanced optimizers such as Adam can also be adopted.
Notably, we will discuss the benefits of gradient descent over Eq. \eqref{eq:single_inr} in Section \ref{sec:implicit_low_rank}.

In summary, the overall workflow of ST-INR are shown in Algorithm \ref{algo:stinr}.

\subsection{Theoretical Analysis}\label{Sec: theoretical_analysis}
In addition to the data learning power of the proposed model, this section presents several salient features that are important for real-world traffic data applications, including \textit{encoding of high-frequency structures}, \textit{implicit low-rank regularization}, \textit{inherent smoothness}, and \textit{reduced complexity}. 
The proposed ST-INR explicitly encodes high-frequency structures to learn complex details of various STTD while at the same time implicitly learning low-rank and smooth priors from data to reconstruct the dominating modes.

\subsubsection{High frequency encodings}
The key rationale for INRs is the exploitation of high-frequency structures in data signals. Based on the theory of neural tangent kernel (NTK) \citep{NTK,FourierFeature}, we first indicate that the Fourier features in Eq. \eqref{eq:crf} introduce a composed NTK that is beneficial for the convergence of neural networks to high-frequency components, providing a solution to the spectral bias of standard MLPs.

\begin{lemma}[Neural network dynamics through NTK \citep{NTK}]\label{lemma:ntk}
 Let $f$ become a deep neural network with parameters $\theta$, the training dynamics of it can be approximated by the NTK defined as: $k_{\text{NTK}}(\mathbf{x}_i,\mathbf{x}_j)=\mathbb{E}_{\theta}\langle\frac{\partial f(\mathbf{x}_i;\theta)}{\partial\theta},\frac{\partial f(\mathbf{x}_j;\theta)}{\partial\theta} \rangle$.   
When considering the input in a hypersphere, the NTK reduces to a dot product kernel $k_{\text{NTK}}(\mathbf{x}_i,\mathbf{x}_j)=h_{\text{NTK}}(\mathbf{x}_i^{\mathsf{T}}\mathbf{x}_j)$.
Then for a MLP trained with an L2 loss and a learning rate $\eta$, the network's output for test data $\mathbf{X}_{\text{test}}$ after $t$ steps of training on a training dataset $\{\mathbf{X},\mathbf{y}\}$ can be approximated as:
\begin{equation}
    \widehat{\mathbf{y}}^{(t)} = f(\mathbf{X}_{\text{test}};\theta)\approx\mathbf{K}_{\text{test}}\mathbf{K}^{-1}(\mathbf{I}-e^{-\eta\mathbf{K}t})\mathbf{y},
\end{equation}
where $\mathbf{K}$ is the kernel matrix between all data pairs of training data with $k_{i,j}=k_{\text{NTK}}(\mathbf{x}_i,\mathbf{x}_j)$, and $\mathbf{K}_{\text{test}}$ is the NTK matrix between all pairs of testing data.
On top of this approximation, we are interested in the behaviors of model in training convergence. Consider the eigendecomposition $\mathbf{K}=\mathbf{Q}\mathbf{\Lambda}\mathbf{Q}^{\mathsf{T}}$, the projected absolute error of training loss is given by:
\begin{equation}
    |\mathbf{Q}^{\mathsf{T}}(\widehat{\mathbf{y}}^{(t)}_{\text{train}}-\mathbf{y})|\approx|\mathbf{Q}^{\mathsf{T}}((\mathbf{I}-e^{-\eta\mathbf{K}t})\mathbf{y}-\mathbf{y})|=|e^{-\eta\mathbf{\Lambda}t}\mathbf{Q}^{\mathsf{T}}\mathbf{y}|.
\end{equation}
This indicates that the absolute training loss will decrease approximately exponentially at the rate $\eta\lambda$, where the larger eigenvalues will be fitted faster than those smaller ones. A kernel with rapidly decaying spectrum fails to converge to high-frequency components. Instead, a wider kernel with slower spectral attenuation can achieve easier convergence.

\end{lemma}

\begin{proof}
    The proof can be found in \citet{NTK}.
\end{proof}

The above lemma states that a standard MLP is difficult to fit high-frequency components, and we can control the spectral bias of deep neural networks through manipulating the kernel to cover a relatively wide spectrum. This is achieved by the CRF in Eq. \eqref{eq:crf} with the following lemma.
\begin{lemma}[Composing NTK with Fourier features \citep{FourierFeature}]\label{lemma:composed_ntk}
Given the Fourier mapping in Eq. \eqref{eq:crf}, the induced NTK has the form:
\begin{equation}
    k_{\gamma}(\gamma(\mathbf{v}_1),\gamma(\mathbf{v}_2))=\gamma(\mathbf{v}_1)^{\mathsf{T}}\gamma(\mathbf{v}_2)=\sum_{j=1}^{N_f}\cos(2\pi\mathbf{B}_j(\mathbf{v}_1-\mathbf{v}_2))\triangleq h_{\gamma}(\mathbf{v}_1-\mathbf{v}_2)),
\end{equation}
which is a shift-invariant kernel, a.k.a., stationary kernel. Then the CRF passed through a MLP activates a composed NTK:
\begin{equation}\label{eq:composed_kernel}
    h_{\text{NTK}}(\mathbf{x}_i^{\mathsf{T}}\mathbf{x}_j)=h_{\text{NTK}}(\gamma(\mathbf{v}_1)^{\mathsf{T}}\gamma(\mathbf{v}_2))=h_{\text{NTK}}(h_{\gamma}(\mathbf{v}_1-\mathbf{v}_2)).
\end{equation}
This result means that Fourier feature mapping of the input coordinates makes the composed NTK stationary, serving as a convolution kernel over the input domain \citep{FourierFeature}.
\end{lemma}

This lemma shows that CRF transforms a kernel from the dot product into a composed stationary one: $k(\gamma(\mathbf{u}),\gamma(\mathbf{v}))=h(\gamma(\mathbf{u})^{\mathsf{T}}\gamma(\mathbf{v}))=h\circ\gamma(\mathbf{u}-\mathbf{v})$, making it better suited for low-dimensional regression.
Notably, regression with a stationary kernel corresponds to data reconstruction with a convolution filter. 

Combining lemmas \ref{lemma:ntk} and \ref{lemma:composed_ntk}, we see that higher frequency mapping results in a composed kernel with a wider spectrum, allowing faster convergence for high-frequency patterns.
By adding a wide range of Fourier features to the input mapping, the resulting composed kernel can have a wider spectrum, serving as an effective convolutional filter for signal reconstruction. In fact, the model predictions generated by Eq. \eqref{eq:composed_kernel} are sums of observed points, weighted by a function of the Euclidean distance, which is demonstrated to be effective for traffic data reconstruction problems, such as the adaptive smoothing method \citep{treiber2013traffic}.

Next, we show that the MLP layers with sine activation functions in Eq. \eqref{eq:siren} are also equivalent to Fourier mappings, thus showing similar effects as the CRF. The following lemma explains this equivalence.
\begin{lemma}[The equivalence between periodic activation and Fourier features]
Given the input coordinate $\mathbf{v}\in\mathbb{R}^{d_{\text{in}}}$, a two-layer MLP with sine activation function and weights $\mathbf{W}_{f}\in\mathbb{R}^{2Nd\times d_{\text{in}}},\mathbf{W}\in\mathbb{R}^{d\times 2Nd}$
is formed as:
\begin{equation}\label{eq:siren_simple}
\phi_s(\mathbf{v})=\mathbf{W}\sin(\omega\mathbf{W}_f\mathbf{v}+\mathbf{b}_f)+\mathbf{b},
\end{equation}
then it can be equivalent to a single-layer linear network with a series of Fourier mapped inputs:
\begin{equation}\label{eq:lemma:fourier}
    h_f(\mathbf{v})=\mathbf{W}\left[
        \cos(2\pi\mathbf{W}_{f_1}\mathbf{v}),
    \sin(2\pi\mathbf{W}_{f_1}\mathbf{v}),\dots,\cos(2\pi\mathbf{W}_{f_N}\mathbf{v}), \sin(2\pi\mathbf{W}_{f_N}\mathbf{v})\right]+\mathbf{b}.
\end{equation}
\end{lemma}

\begin{proof}
We prove this lemma by construction. Without generality, we denote: $\mathbf{W}_f=[\mathbf{W}_{f_1},\mathbf{W}_{f_1},\dots,\mathbf{W}_{f_N},\mathbf{W}_{f_N}]^{\mathsf{T}}\in\mathbb{R}^{2Nd\times d_{\text{in}}}$ with $\mathbf{W}_{f_i}\in\mathbb{R}^{d\times d_{\text{in}}}$ and $\mathbf{b}_f=[\pi/2,0,\dots,\pi/2,0]\in\mathbb{R}^{2Nd}$. Then Eq. \eqref{eq:siren_simple} can be rewritten as:
\begin{equation}
    \mathbf{W}\sin(\omega\left[\begin{aligned}
        \mathbf{W}_{f_1}&\mathbf{v},\\
        \mathbf{W}_{f_1}&\mathbf{v},\\
        &\vdots \\
        \mathbf{W}_{f_N}&\mathbf{v}, \\
        \mathbf{W}_{f_N}&\mathbf{v}
    \end{aligned}\right]
    +\left[\begin{aligned}
    \pi&/2,\\ 
    &0,\\
    \vdots&\\
    \pi&/2,\\
    &0
    \end{aligned}\right]
    )+\mathbf{b}~=\mathbf{W}\left[\begin{aligned}
    \cos(\omega\mathbf{W}_{f_1}&\mathbf{v}), \\
    \sin(\omega\mathbf{W}_{f_1}&\mathbf{v}), \\
    \vdots\\
    \cos(\omega\mathbf{W}_{f_N}&\mathbf{v}), \\
    \sin(\omega\mathbf{W}_{f_N}&\mathbf{v}), \\
    \end{aligned}\right]+\mathbf{b}.
\end{equation}
When $\omega=2\pi$, the above equation equals \eqref{eq:lemma:fourier}.
\end{proof}

The two techniques show similar effects in both the input and hidden layers of ST-INR and enable it to encode high-frequency structures adaptively to learn complex data patterns, which is a significant superiority over standard low-rank models.

\subsubsection{Implicit low-rank regularization} \label{sec:implicit_low_rank}
In addition to high-frequency details, STTD is dominated by a few low-rank structures \citep{tan2013tensor,asif2016matrix}. Conventional low-rank models such as matrix factorization and nuclear norm minimization capture the low-rank structures with explicit low-rank constraints. As an alternative, we show that our model features an implicit low-rank regularization that is learned from the gradient descent over the target data. 

Recall that $d_x$ and $d_t$ in Eq. \eqref{eq:factorize} determines the dimension of the interaction between spatial and temporal factors. In canonical matrix factorization or tensor factorization methods, such a dimension is set to a very small number $r\ll\min\{N,T\}$ with $N,T$ denoting the dimension of the input matrix, to encourage a low-rank solution. It is widely acknowledged in the field that low-rankness significantly benefits the reconstruction of traffic data \citep{tan2013tensor}. However, the set $d_x=d_t=r$ can restrict the model expressivity of $\Phi_x(x)$ and $\Phi_t(t)$ that are parameterized by deep neural networks. A compromise needs to be reached between explicit low-rank regularization and high model capacity. Additionally, traffic data from different sources and scales can have distinct low-rank patterns. For example, the low-rank pattern identified in microscopic data such as vehicle trajectory can not generalize to macroscopic data such as mobility flow. The determination of different rank values for various STTD can pose a great challenge.

Fortunately, we can address this issue by the implicit low-rank regularization of deep models. Specifically, we treat our model as a special case of deep matrix factorization (DMF) \citep{arora2019implicit}. We can ensure both the low-rankness and model capacity by adopting a full-dimensional factor matrix, i.e., $d_x=d_t=\min\{N,T\}$ and imposing an implicit low-rank regularization from the gradient descent method.

To analyze the low-rankness of our model, on the one hand, we can impose some structural constraints on Eq. \eqref{eq:factorize}. We characterize $\Phi$ as a dynamical system that is conditioned on the evolution step $\tau$ of gradient descent. We assume $\Phi_x(\tau)\in\mathbb{R}^{N\times\min\{N,T\}},\Phi_t(\tau)\in\mathbb{R}^{T\times\min\{N,T\}}$ have orthonormal columns and $\mathbf{M}(\tau)\in\mathbb{R}^{\min\{N,T\}\times\min\{N,T\}}$ is diagonal. Then $\Phi(\tau)=\Phi_x(\tau)\operatorname{diag}(\sigma_1(\tau),\dots,\sigma_{\min\{N,T\}}(\tau))\Phi_t(\tau)^{\mathsf{T}}$ with $\sigma_j(\tau)$ denotes the signed singular values of $\Phi(\tau)$. On the other hand, if we omit the activation and bias terms, we have $\Phi(\mathbf{v})=\mathbf{W}_x^{L+1}\mathbf{W}_x^{{L}}\cdots\mathbf{W}^0_x\mathbf{v}\mathbf{M}\mathbf{W}_t^{0,\mathsf{T}}\cdots\mathbf{W}_t^{L+1,\mathsf{T}}$, which forms a special DMF model defined in \citep{arora2019implicit}.
Under these conditions, our factorized model invokes the following lemma.

\begin{lemma}[Implicit low-rank regularization of DMF \citep{arora2019implicit}]\label{lemma:low-rank}
Given the loss function $\mathcal{L}$ in Eq. \eqref{eq:single_inr}, the signed singular values $\sigma(\tau)$ of $\Phi_{\tau}$ evolve by the following rule:
\begin{equation}
    \dot{\sigma}_j(\tau)=-2L(\sigma_j^2(\tau))^{1-\frac{1}{2L}}\langle \nabla\mathcal{L},\Phi_x^j(\tau)\Phi_t^j(\tau)^{\mathsf{T}}\rangle,~j=1,\dots\min\{N,T\},
\end{equation}
where $\Phi_x^j(\tau),\Phi_t^{j}(\tau)$ are the $j$-th column vectors of $\Phi_x,\Phi_t$, and $L$ is the depth of $\mathcal{F}_{\theta}$.
\end{lemma}
\begin{proof}
    The proof can be found in Theorem 3 of \citet{arora2019implicit}.
\end{proof}

Lemma \ref{lemma:low-rank} indicates that the evolution of singular values of $\Phi$ conditions on the depth $L$ of the factorized model. When $L\geq 1$, the nontrivial factor $-2L(\sigma_j^2(\tau))^{1-\frac{1}{2L}}$ intensifies the evolution of large singular values and attenuates that of smaller ones. Such an updating rule promotes a factorization that features a small number of large singular values and a large number of small ones. In addition, a larger $L$ encourages more significant gaps between large and small singular values.
This is an implicit bias (regularization) towards a low-rankness. With this, we can ensure the expressivity of MLPs by setting a large width and depth and obtain a low-rank resolution by the regularization of gradient descent. Moreover, we can directly adopt full-dimensional factorization to bypass the need to select a task-specific rank $r$.
Demonstrations of this property are given in Section \ref{exp:low-rank}.

\subsubsection{Inherent smoothness from MLPs}\label{Sec: smoothness}
Apart from the low-rankness, some other structural priors also have significant impacts on STTD.
For example, smoothness priors are essential to reconstruct traffic data under poor observation conditions \citep{chen2022laplacian,nie2023correlating}. As revealed in previous work, temporal smoothness (continuity) is helpful for estimating data from missing time intervals, and graph spectral or spatial smoothness can be exploited for extrapolation. However, most of these handcrafted priors are task-dependent and may not be generalized to other scenarios. Therefore, we resort to the inherent smoothness property of our model that is scenario-independent and fully learned from the data.
Specifically, \cite{luo2023low} has shown that, under some assumptions, the tensor function is Lipschitz continuous in Euclidean space. As a complement, we examine its continuity in arbitrary spectral domain, i.e., in a topological space defined in Section \ref{Sec: graph-embedding}.

\begin{lemma}[Spectral smoothness of ST-INR]\label{lemma:smoothness}
Consider a single ST-INR defined on an arbitrary topological space: $\Phi(\mathbf{e})=\mathbf{W}^L(\sigma(\mathbf{W}^{{(L-1)}}(\cdots\sigma(\mathbf{W}^1\mathbf{e}))))$, we assume that the $\ell_1$ norm of each INR weight matrix $\mathbf{W}^{(\ell)}$ is bounded by $\xi$, and the $\ell_1$ norm of middle transform $\mathcal{M}$ is bounded by $\eta$, then we have:
\begin{equation}\label{eq:lipschitz}
    \begin{aligned}
        |\Phi(\mathbf{e}_1,\mathbf{e}_2,\dots,\mathbf{e}_n)-\Phi(\mathbf{e}_1',\mathbf{e}_2,\dots,\mathbf{e}_n)|&\leq\eta\xi^{nL}\delta^{n-1}|\mathbf{e}_1-\mathbf{e}_1'|,\\
        |\Phi(\mathbf{e}_1,\mathbf{e}_2,\dots,\mathbf{e}_n)-\Phi(\mathbf{e}_1,\mathbf{e}_2',\dots,\mathbf{e}_n)|&\leq\eta\xi^{nL}\delta^{n-1}|\mathbf{e}_2-\mathbf{e}_2'|,\\
        &\vdots \\
        |\Phi(\mathbf{e}_1,\mathbf{e}_2,\dots,\mathbf{e}_n)-\Phi(\mathbf{e}_1,\mathbf{e}_2,\dots,\mathbf{e}_n')|&\leq\eta\xi^{nL}\delta^{n-1}|\mathbf{e}_n-\mathbf{e}_n'|,\\
    \end{aligned}
\end{equation}
which means that $\Phi(\mathbf{e})$ is Lipschitz continuous in an arbitrary spectral coordinate system $\mathbf{e}$.
\end{lemma}

\begin{proof}
According to the factorization rule in Eq. \eqref{eq:factorize}, we have $\Phi(\mathbf{e}_1,\mathbf{e}_2,\dots,\mathbf{e}_n)=\mathcal{M}\times_1\Phi(\mathbf{e}_1)\times_2\Phi(\mathbf{e}_2)\cdots\times_n\Phi(\mathbf{e}_n)$, where $\times_n$ is the tensor product \citep{kolda2009tensor}. For $\forall \Phi(\mathbf{e})$, $|\Phi(\mathbf{e})|\leq |\mathbf{W}^L||\sigma(\mathbf{W}^{L-1}\cdots\sigma(\mathbf{W}^{1}))||\mathbf{e}|$, and both $\sin(\cdot)$ and $\texttt{ReLU}(\cdot)$ are Lipschitz continuous everywhere with Lipschitz constant equals to 1, then it holds:
\begin{equation}\label{eq:proof1}
    \begin{aligned}
        |\Phi(\mathbf{e}_1,\mathbf{e}_2,\dots,\mathbf{e}_n)-\Phi(\mathbf{e}_1',\mathbf{e}_2,\dots,\mathbf{e}_n)|&= |\mathcal{M}\times_1(\Phi(\mathbf{e}_1')-\Phi(\mathbf{e}_1'))\times_2\Phi(\mathbf{e}_2)\cdots\times_n\Phi(\mathbf{e}_n)|,\\
        &\leq \eta\xi^{(n-1)L}\Pi_{i\neq 1}|\mathbf{e}_i|(\Phi(\mathbf{e}_1')-\Phi(\mathbf{e}_1)),\\
        &\leq  \eta\xi^{nL}\Pi_{i\neq 1}|\mathbf{e}_i||\mathbf{e}_1'-\mathbf{e}_1|.\\
    \end{aligned}
\end{equation}
Eq. \eqref{eq:proof1} holds for $n$ coordinates. If we let $\delta=\max\{|\mathbf{e}_i|,i=1,\dots,n\}$, Eq. \eqref{eq:lipschitz} naturally holds.
\end{proof}

Lemma \ref{lemma:smoothness} indicates that our model is continuity in arbitrary spectral coordinate system. This property enables our model to impose smoothness priors on graphs and beyond the regular $x-t$ coordinates, which aligns with the characteristics of network-wide traffic data. We illustrate this feature in Section \ref{sec:exp_network}. With this property, we can implicitly regularize the smoothness of the solution through governing the continuity of MLPs, without having to elaborate on a complicated penalty function. In practice, this can be achieved by controlling the weight decay parameters of MLPs, refer to Section \ref{sec:hyperparameter}.

\subsubsection{Reduced computational complexity}\label{sec:complexity}
Due to the factorized design, our model features reduced computational complexity. Besides, it does not require resolution-dependent quadratic or cubic storage, such as the factor matrix and the core tensor, in low-rank models. This makes it applicable for learning large-scale traffic data instance.
We examine the complexity of each forward pass in two input settings: (1) continuous domain and (2) regular grids.

For case (1), we consider that $n$ observation points in a two-dimensional space are measured for training. The model inputs two separate coordinates $[n,1]\times[n,1]$ and maps each of them to $[n,d]$ where $d$ is the output size. Then two factors are merged to compute the final output $[n,1]$.
In this situation, the time complexity of a $L$ layer ST-INR with hidden size $D$ is $\mathcal{O}(nLD^2+nd^2+nDd)$. 

For case (2), the input is a mesh grid with size $n_1\times n_2$. Without factorization, the complexity becomes $\mathcal{O}(n_1n_2LD^2+n_1n_2d^2+n_1n_2Dd)$. When spatial-temporal distanglement is adopted, the complexity reduces to $\mathcal{O}((n_1+n_2)LD^2+(n_1+n_2)Dd)$, which is much more efficient for large-scale data with $n_1+n_2\ll n_1n_2$. 

As for space complexity, general MF models need to store the entire factor matrices $\mathbf{U}\in\mathbb{R}^{n_1\times d},\mathbf{V}\in\mathbb{R}^{n_2\times d}$ and in the running memory to update each iteration. Instead, our model only needs to store some weight matrix after model training, e.g., $\mathbf{W}_1\in\mathbb{R}^{1\times D},\mathbf{W}_L\in\mathbb{R}^{D\times d}$, which is agnostic to the input dimensions, which makes our model scalable for large input dimensions. 

We will show the computational superiority of our factorization strategy in Section \ref{exp:efficiency}.

\section{Demonstrations of Different Applications}\label{Sec: experiments}
This section demonstrates the practical effectiveness of the proposed ST-INR model. Several experiments were carefully designed on public STTD datasets, covering scales ranging from highway corridors, urban grids, to network levels. We first compare it with some representative low-rank models on benchmark tasks, then we perform supplementary studies to highlight the unique properties of our model.
All experiments are conducted on a computing platform with a single NVIDIA GeForce RTX A6000 GPU (48 GB). \textbf{PyTorch implementations of this project is publicly available at:} \url{https://github.com/tongnie/traffic_dynamics}.

% \subsection{Application 1: traffic data representation learning (expressivity)}
% (a) Reconstruct fine-detailed, high-order, sparse, and spatial traffic data, such as OD map, time series.
% \subsubsection{Learning compact representations for high-dimensional traffic data}
% Compared to the LRMC model with the same rank parameter, our model has much fewer parameters.

% data: TaxiBJ, NYC, METR-LA, Uber, Traffic4cast, T-Drive trajectory data.

% \subsubsection{Learning location-aware representations for large-scale network}
% \textcolor{purple}{It can inference at arbitrary locations and time intervals within a continuous space.}
% \textcolor{purple}{Location-aware applications? estimate the network-wide traffic state.}

% \subsection{Single-scale seperate modeling}
\subsection{Corridor-level application: Highway traffic state estimation}
We first consider the traffic state estimation (TSE) problem in highway corridors. 
TSE is a prerequisite procedure for accurate reconstruction of individual trajectories \citep{chen2024macro}. Traditional methods include traffic flow models such as the Lighthill-Whitham-Richards model and smoothing-based models such as the adaptive smoothing method \citep{treiber2013traffic}. Actually, if we organize the spatial-temporal speed field into regular meshes, this can be treated as a structured matrix completion problem \citep{wang2023low}.

We adopt the trajectory data from NGSIM, US 101 dataset \footnote{\url{https://www.fhwa.dot.gov/publications/research/operations/07030/index.cfm}.} for experiment.
{This is an open-source dataset from the Next Generation simulation program on southbound US highway 101 of lane 2. Traffic trajectories were collected using digital video cameras. 
To create the TSE problem, we select an area in the second lane with a length of 1500 ft and a time span of 2600 s, resulting in a total of approximately 1300 vehicles.
Each trajectory point records its traffic speed value.
} Two input scenarios are considered: (a) estimation from discrete mesh grids; (b) estimation from continuous trajectories. 
In the first setting, we need to split the continuous spatial-temporal domain into regular mesh grids, to ensure the function of discrete baselines. Therefore, we use a time resolution of 5 s and a space resolution of 7 ft, resulting in a matrix of dimension $(215,520)$. We consider the following baselines: (1) Matrix factorization with alternating least squares (MF); (2) Low-rank matrix completion based on nuclear norm minimization (LRMC); (3) Laplacian convolutional representation-based matrix completion (LCR) \citep{chen2022laplacian}; (4) Circulant nuclear norm minimization model (CircNNM) \citep{liu2022recovery}; (5) STHTC: spatiotemporal Hankelized tensor completion \citep{wang2023low}, which is a state-of-the-art TSE model; (6) ASM: Adaptive smoothing method \citep{treiber2013traffic}; (7) MLP: vanilla MLP model.
We randomly sampled $15\%$ of the trajectories as observed probe vehicles, and all models need to estimate the full speed field.
The input to ST-INR is the discrete index of each observed cell, and the observed trajectory points are used for training. The trained ST-INR are adopted to predict the traffic states at all meshes.
% In this scenario, we turn off the meta-learning procedure in Algorithm \ref{meta-learning} and train a single INR to represent the whole data.

\begin{table}[!htb]
  \centering
  \caption{Results (in terms of WMAPE, RMSE, and MAE) of highway traffic state estimation on mesh grid.}
    \begin{tabular}{c|>{\columncolor{blue!20!white}}cccccccc}
    \toprule
    \multicolumn{1}{c}{Metric}  & \multicolumn{1}{c}{\texttt{ST-INR}}& \multicolumn{1}{c}{\texttt{MF}}  & \multicolumn{1}{c}{\texttt{LRMC}} & \multicolumn{1}{c}{\texttt{LCR}} & \multicolumn{1}{c}{\texttt{CircNNM}} & \multicolumn{1}{c}{\texttt{STHTC}} & \multicolumn{1}{c}{\texttt{ASM}} & \multicolumn{1}{c}{\texttt{MLP}}\\
    \midrule
    WMAPE & \textbf{11.05\%} &21.70\% & 20.60\% &12.21\% & 12.72\%& 12.06\%& 13.57\%& 29.29\%\\
    RMSE (ft/s)  & \textbf{4.88} & 9.96& 9.62& 5.48&5.74 & 5.46&5.94 &12.51\\
    MAE (ft/s)  & \textbf{3.71} & 7.29& 6.92& 4.10 & 4.27& 4.05& 4.56&9.84\\
    \bottomrule
    \multicolumn{5}{l}{\scriptsize{Best performances are bold marked.}}
    \end{tabular}%
  \label{Tab: ngsim_results}%
\end{table}%

Results on discrete grids are given in Tab. \ref{Tab: ngsim_results} and Fig. \ref{fig:ngsim_result}. As can be seen, our model outperforms competing models by a large margin in this benchmark task. As indicated in Fig. \ref{fig:ngsim_result} (especially the black box in the lower right corner), our model can reconstruct fine details with high resolution, which is more consistent with the real congestion shock wave.
Interestingly, the vanilla MLP fails to learn any congestion wave and approximates the average speed distribution. This result clearly demonstrates that the high-quality reconstruction by ST-INR can be ascribed to the utilization of high-frequency features.

In the second scenario, we evaluate our model directly on continuous trajectory points, in which all existing low-rank baselines are incapable of working.
We randomly sampled $10\%$ of the $(x,t)$ coordinates as the collocation points to fit the continuous mapping, leading to approximately $5$ thousands of points.
The input to ST-INR is a set of $(x,t)$ coordinates. We compare our model with a MLP model equipped with sine activation function \citep{SIREN}. {We first trained the models on sparse trajectory points and enable them to predict the speed values of all trajectory points for evaluation. We record the training and generalization loss in Tab. \ref{Tab: ngsim_result_continuous}. We find that the pure MLP model can fit the observed points with acceptable accuracy, but it has difficulty generalizing to unseen data points.}

It is noteworthy that after being trained on the sampled collocation points, INRs can make inference at any resolution and position. Therefore, we upsample the trajectory points with a higher resolution. We uniformly sample $1600$ points along the time axis and $800$ points along the location axis, resulting in a $245\times$ upsampling rate. As shown in Fig. \ref{fig:ngsim_continuous_result}, our model can generate a consistent speed contour and successfully reconstruct the traffic flow phenomena, e.g., the shockwave propagation and the free-flow speed.

\begin{figure}[!htb]
\centering
\includegraphics[scale=0.32]{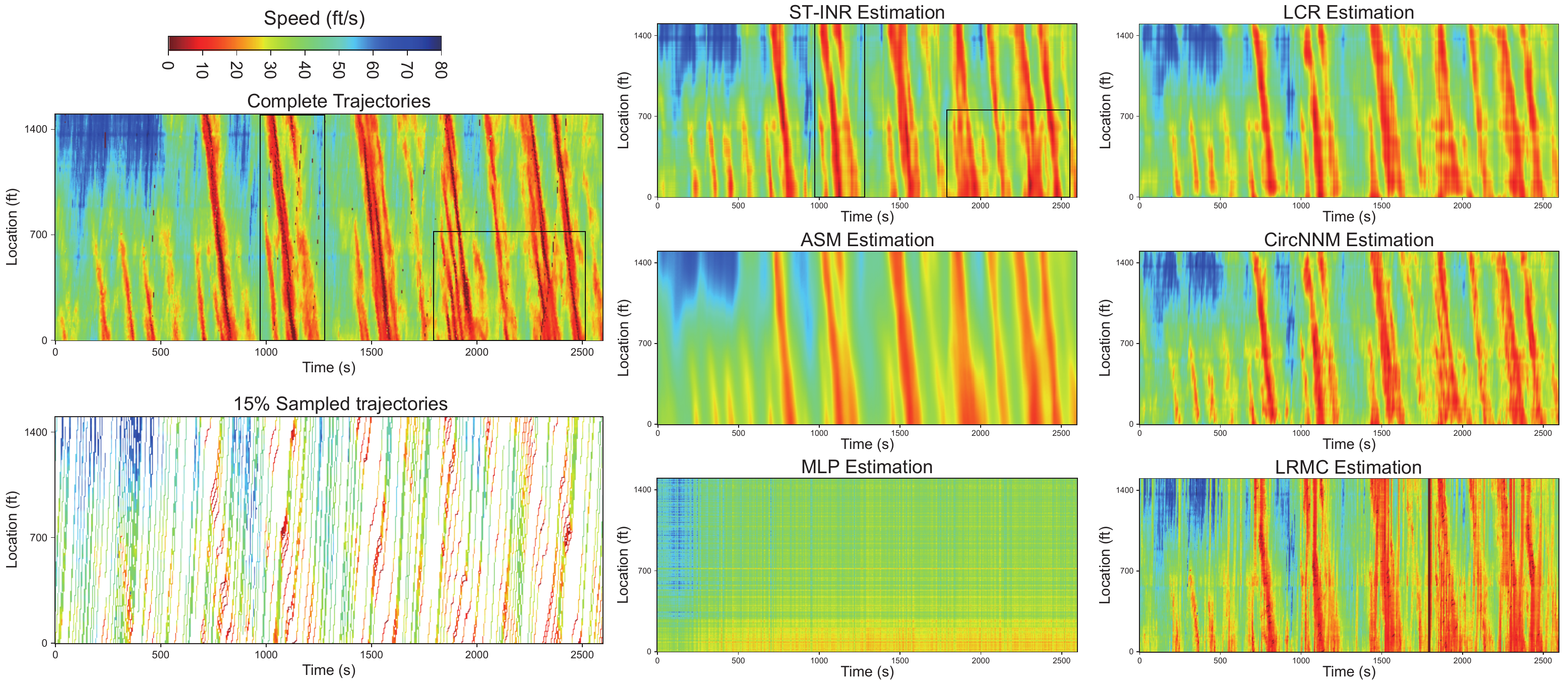}
\caption{TSE performances on discrete grid.}
\label{fig:ngsim_result}
\end{figure}

\begin{table}[!htb]
  \centering
  \caption{{Results of training and generalization performances on continuous trajectory points.}}
    \begin{tabular}{c|cc|cc}
    \toprule
    \multirow{2}{*}{Metric} & \multicolumn{2}{c|}{{Training loss}} & \multicolumn{2}{c}{{Generalization loss}} \\
    \cmidrule{2-5}
    & {\texttt{ST-INR}}& {\texttt{MLP}}  & {\texttt{ST-INR}} & \multicolumn{1}{c}{\texttt{MLP}} \\
    \midrule
    WMAPE & \textbf{3.98\%} & 5.23\% & \textbf{6.28\%} & 17.90\% \\
    RMSE (ft/s)  & \textbf{1.73} & 2.15 & \textbf{2.85} & 7.71\\
    MAE (ft/s)  & \textbf{1.29} & 1.69& \textbf{2.03} & 5.79\\
    \bottomrule
    \multicolumn{5}{l}{\scriptsize{Best performances are bold marked.}}
    \end{tabular}%
  \label{Tab: ngsim_result_continuous}%
\end{table}%

\begin{figure}[!htb]
\centering
\subfigure[Continuous trajectory points]{
\includegraphics[scale=0.32]{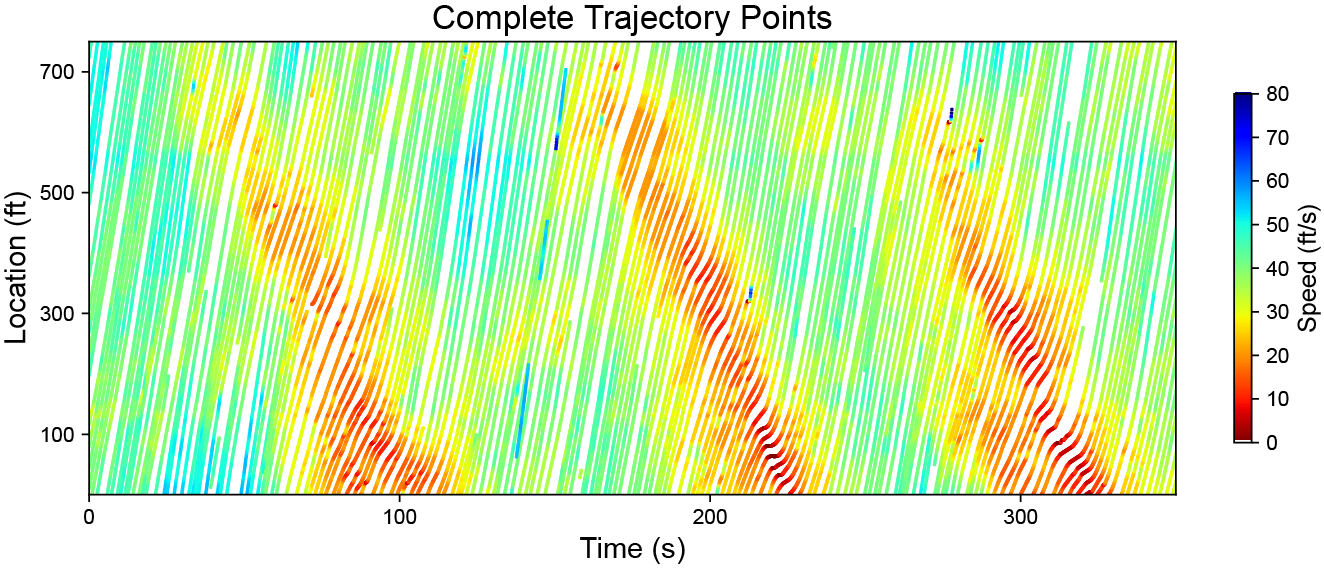}}
\centering
\subfigure[Sparsely sampled trajectory points]{
\includegraphics[scale=0.32]{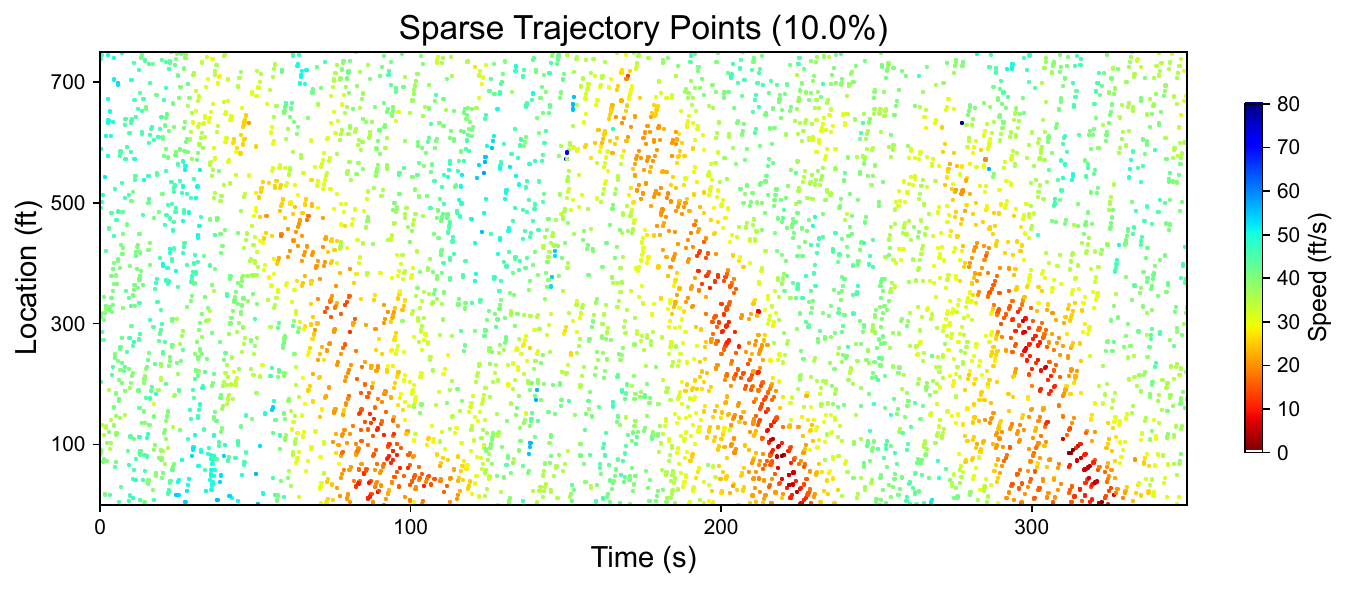}}
\centering
\subfigure[Upsampled trajectory points by ST-INR]{
\includegraphics[scale=0.35]{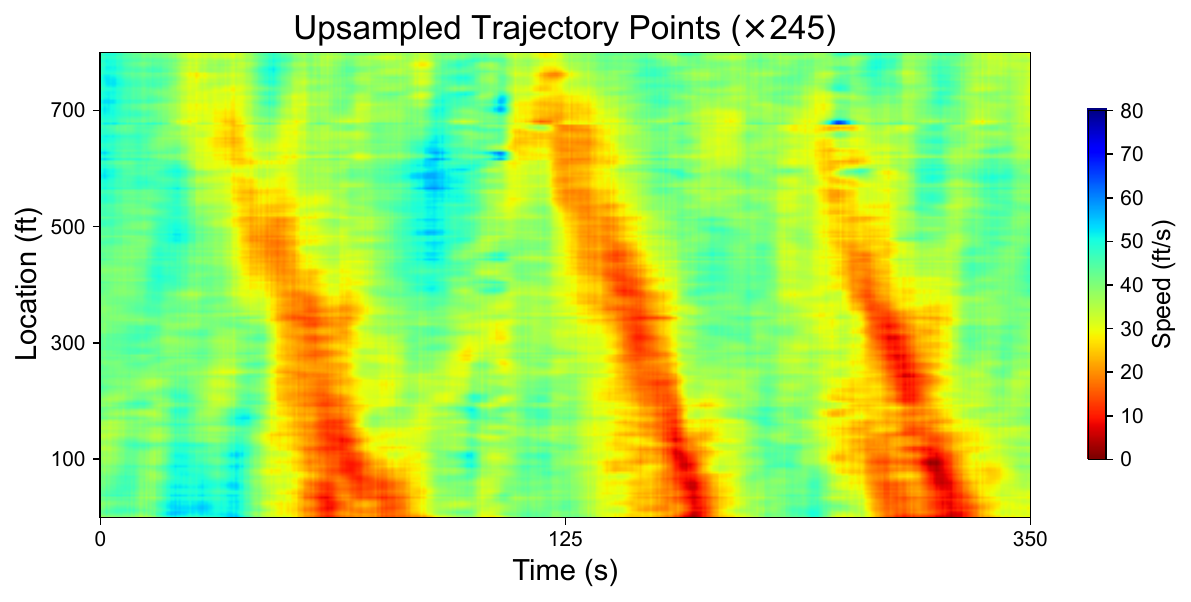}}
\centering
\subfigure[Upsampled trajectory points by MLP (Sine)]{
\includegraphics[scale=0.35]{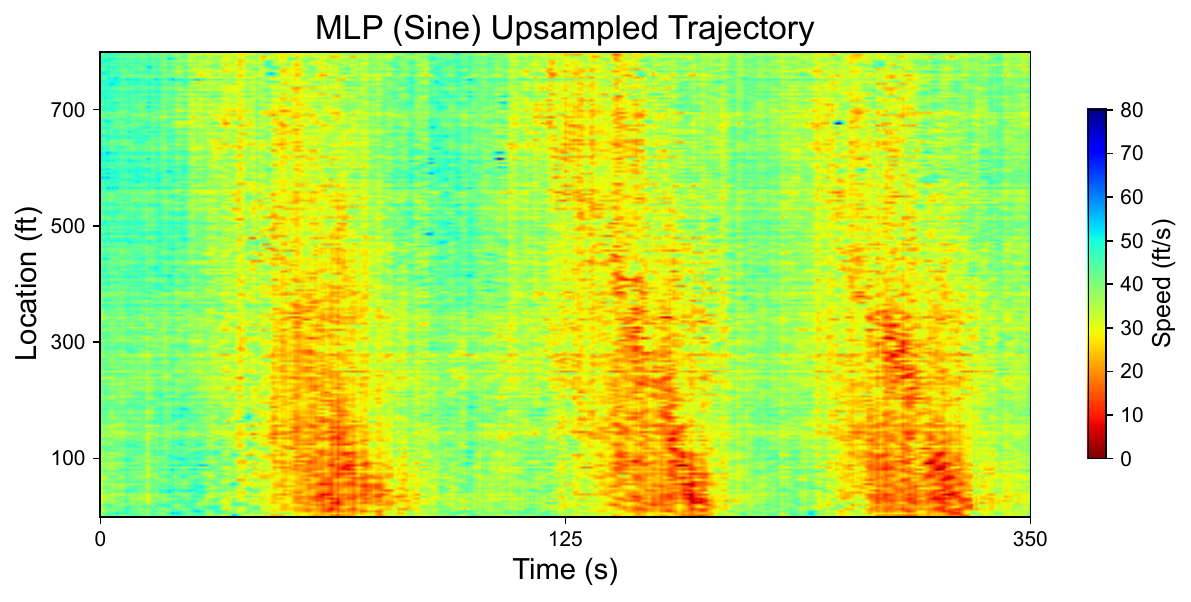}}
\caption{TSE performances on continuous space.}
\label{fig:ngsim_continuous_result}
\end{figure}

% \subsection{Application 3: highway traffic dynamics reconstruction (smoothness)}

% data: Seattle speed, Portland/PeMS speed

% \textcolor{purple}{When using matrix or tensor models for kriging (extrapolation), one first need to construct a graph or adjacency matrix to describe the physical connection between all locations. They are agnostic to the absolute locations of the sensor. Instead, INRs are location-aware.}

\subsection{Grid-level application: Urban mesh-based flow estimation}
The second task evaluates the estimation of the urban mesh-based flow map. This task can have wide applications in measuring and regulating mobility flow in urban areas \citep{liang2019urbanfm}. The grid flow map can be arranged with two spatial dimensions to denote the locations of the gird or origin-destination locations, and with a temporal dimension to denote the time. This structure can be  processed by a third-order tensor.

Two public datasets are adopted: (1) TaxiBJ \footnote{\url{https://github.com/yoshall/UrbanFM}.}: {grid-based taxi volume records in the main district of Beijing, where each grid reports the flow information at a $30$-min interval from March 1 to June 30, 2015. 
This data consists of trajectory data from taxicab GPS data and meteorology data in Beijing from four time intervals, which is published by \citep{zhang2017deep}.}
We keep the first 7 days and use the spatial resolutions of $32\times32$ and $64\times64$, resulting in tensors with sizes $(336,32,32)$ and $(336,128,128)$; (2) TaxiNYC \footnote{\url{https://www.nyc.gov/site/tlc/about/tlc-trip-record-data.page}.}: taxi trip records in NYC. {This datasets are from the New York City Taxi and Limousine Commission (TLC), capturing trip records for yellow taxis. The datasets encompass various fields such as pick-up and drop-off dates/times, locations, trip distances, fares, rate types, payment methods.}
Similarly to previous work, we use 69 zones in Manhattan district as pick-up and drop-off zones. We aggregate the daily taxi trip into origin-destination matrix and concatenate the mobility data from January to May in 2020 to form a tensor of size $(69,69,151)$. Due to the impact of COVID-19, there is a drastic drop in travel demands after March 2020 (see the first row in Fig. \ref{fig: nyc_result}).
% For TaxiBJ data, we consider two scenarios: (1) separate estimation of different resolutions; and (2) joint estimation of two resolutions with a single model. 
For TaxiBJ data, we train separate models to estimate volumes at different resolutions.
We use $20\%$ of the data as observed flow and train all models to predict the remaining volumes. For TaxiNYC data, due to the sparsity of origin-destination flows, we randomly sample $40\%$ of the mobility data as training data.
For baselines, we consider: (1) CP: Canonical polyadic tensor decomposition; (2) Tucker: Tucker tensor decomposition; (3) LATC: Low-rank autoregressive tensor completion \citep{chen2021Autoregressive}, which is a state-of-the-art model for tensor-structured time series data; (4) LRTC-TNN: low-rank tensor completion based on truncated tensor nuclear norm \citep{chen2020nonconvex}; (5) Vanilla MLP model.

Qualitative and quantitative results are given in Tab. \ref{Tab: BJ_NYC_OD_results} and Figs. \ref{fig:taxibj_result} and \ref{fig: nyc_result}. First, Tab. \ref{Tab: BJ_NYC_OD_results} shows that our mode significantly surpasses other baselines, even without explicit temporal modules such as the autoregression. Note that ST-INR depends on the absolute locations of the data points, which is beneficial for the estimation of location-dependent flow patterns. Second, by observing Fig. \ref{fig:taxibj_result}, we find that ST-INR can capture the main shape of the Beijing road network and perform well in both resolutions. While tensor-based models can generate blur estimations without distinguishing the background area. MLPs cannot complete this task because of the inability to learn high-frequency patterns. Finally, Fig. \ref{fig: nyc_result} compares the results of ST-INR and the tensor model on NYC data. Clearly, our model readily learns complex mobility patterns, but the Tucker model can have difficulty learning accurate flows in many zones. In addition, due to the sparsity of mobility flows, tensor-based models can overfit zero values and generate incorrect flows for zones with small values.

\begin{table}[!htb]
  \centering
  \caption{Results of urban grid flow map estimation.}
    \begin{tabular}{c|c|>{\columncolor{blue!20!white}}cccccc}
    \toprule
    \multicolumn{2}{c}{Metric}  & \multicolumn{1}{c}{\texttt{ST-INR}} & \multicolumn{1}{c}{\texttt{CP}} & \multicolumn{1}{c}{\texttt{Tucker}} & \multicolumn{1}{c}{\texttt{LATC}} & \multicolumn{1}{c}{\texttt{LRTC-TNN}} & \multicolumn{1}{c}{\texttt{MLP}} \\
    \midrule
    \multirow{2}[1]{*}{TaxiBJ $32\times 32$} 
          & RMSE (veh/30min) & \textbf{62.11} & 95.86 & 96.35 & 77.17 & 73.27 & 237.86  \\
          & WMAPE & \textbf{14.93\%} & 24.43\% & 24.23\%& 17.40\% & 16.79\% & 55.83\%  \\
    \midrule
    \multirow{2}[1]{*}{TaxiBJ $64\times 64$} 
          & RMSE (veh/30min) & \textbf{21.64} & 33.06 & 32.45 & 27.35 & 26.62 & 98.77  \\
          & WMAPE & \textbf{17.44\%} & 29.15\% & 28.80\% & 22.51\% & 22.34\% & 74.71\% \\
        \midrule
    \multirow{2}[1]{*}{TaxiNYC} 
          & RMSE (veh/d) & \textbf{8.71} & 15.34 & 15.23 & 11.25 & 12.10& 62.86  \\
          & WMAPE & \textbf{14.53\%} & 22.28\% & 21.24\% & 16.87\% & 17.28\% & 70.22\% \\
    \bottomrule
    \multicolumn{5}{l}{\scriptsize{Best performances are bold marked.}}
    \end{tabular}%
  \label{Tab: BJ_NYC_OD_results}
\end{table}%

\begin{figure}[!htb]
\centering
\includegraphics[scale=0.32]{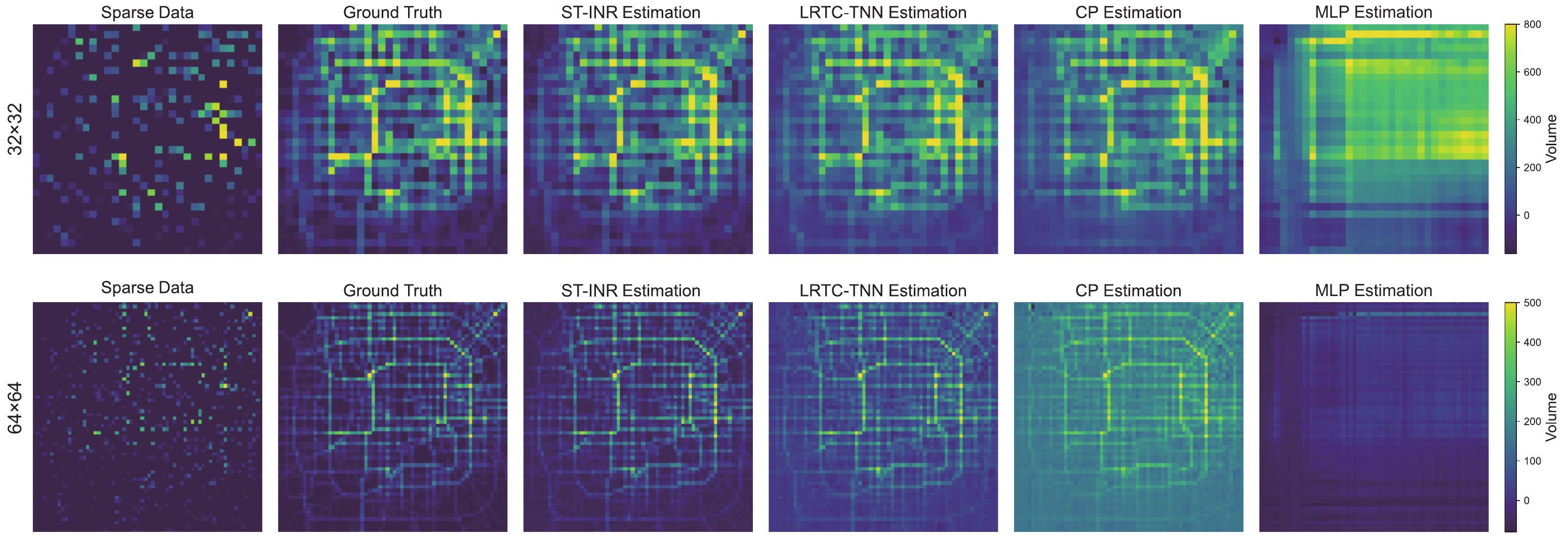}
\caption{Model performances on TaxiBJ data (the 10-th snapshot).}
\label{fig:taxibj_result}
\end{figure}

\begin{figure}[!htb]
\centering
\includegraphics[scale=0.5]{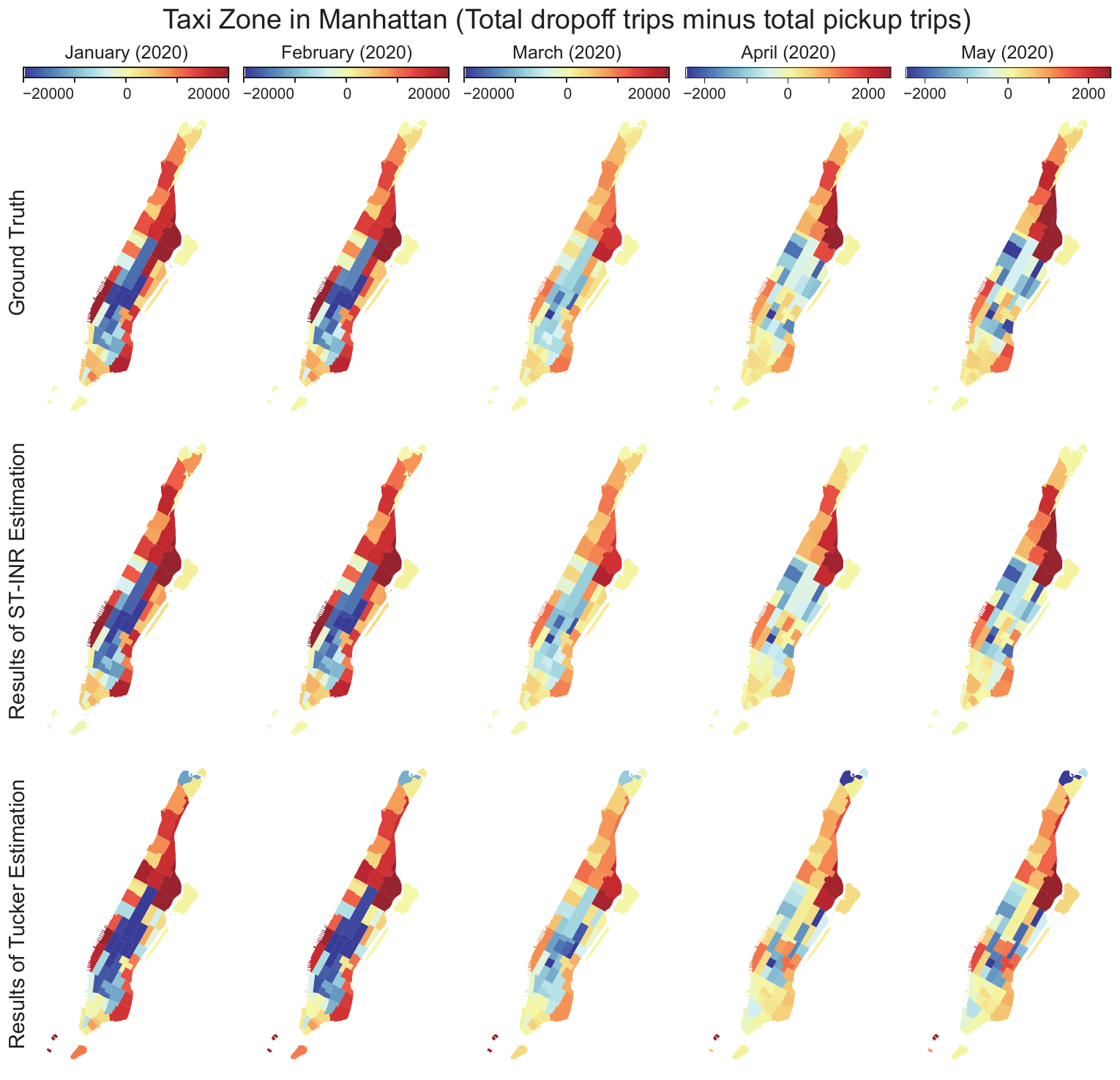}
\caption{Model performances on TaxiNYC data (Jan.-May., 2020).\protect\footnotemark}
\label{fig: nyc_result}
\end{figure}

\footnotetext{The visualization of this figure is based on the code from: \url{https://github.com/xinychen/vars}.}

\subsection{Network-level application: Highway and urban network state estimation}\label{sec:exp_network}
Next, we focus on estimating the states of traffic on the entire road network. The real-time states of the road network are vital for network-wide control and analysis \citep{saeedmanesh2021extended,paipuri2021empirical}. Unlike the above data structures, network data can sometimes be organized as a graph. Therefore, it is challenging for models that are defined on Euclidean space to model non-Euclidean data. In this case, we use the method in Section \ref{Sec: graph-embedding} to project the input coordinates to the Laplacian space. We use two variations of ST-INR with respect to the definition space, including: ST-INR (L): ST-INR defined in Laplacian eigenvector space, and ST-INR(E): ST-INR defined in Euclidean space. We train both of them to predict the adopted data.

For highway sensor network, we adopt: (1) PeMS-BAY: highway loop speed data from $325$ static detectors in the San Francisco South Bay Area collected by the Caltrans Performance Measurement System (PeMS) in 5 minutes \footnote{\url{https://pems.dot.ca.gov/}.}, where the first 30 days of data are used in our experiment; and (2) Seattle: speed data collected by 323 loop detectors deployed on freeways in Seattle area \footnote{\url{https://github.com/zhiyongc/Seattle-Loop-Data}.}. We also use a month of data in a 5-min interval.
For urban road network, we consider the Chicago and Berlin network from the NeurIPS2021-traffic4cast Competition \footnote{\url{https://developer.here.com/sample-data}.}, which transforms link-based traffic speed values to gray values and organizes the network states at each time into an ``image'' with size $(495,436)$. We adopt 288 frames to consider the within-day variations.

We create a ``missing sensor'' estimation problem for the two highway datasets. This task aims to estimate the speed at locations without loop sensors, which is viewed as a spatial kriging problem in the literature \citep{wu2021inductive,nie2023correlating}. We randomly remove a proportion of sensors ($60\%$ for PEMS-BAY and $80\%$ for Seattle) as unmeasured and train models to predict the unmeasured speed values using the adjacency between the sensors. To compare with methods capable of performing this task, we adopt baselines that are regularized either by graphs or smoothness priors: (1) SMF: smooth matrix factorization with total variation penalty \citep{he2015total}; (2) GRALS: graph-regularized alternating least squares \citep{rao2015collaborative}; (3) LCR; (4) GLTL: greedy low-rank tensor learning with Laplacian regularization \citep{bahadori2014fast}; (5) FP: a state-of-the-art graph diffusion-based model to estimate missing graph labels, called feature propagation \citep{rossi2022unreasonable}. The results of model comparison are shown in Tab. \ref{Tab: kriging_results} and Fig. \ref{Fig:highway_result}.

\begin{table}[!htb]
  \centering
  \caption{Results of highway network-wide traffic state estimation.}
    \begin{tabular}{c|c|>{\columncolor{blue!20!white}}c>{\columncolor{blue!20!white}}c cccccc}
    \toprule
    \multicolumn{2}{c}{Metric}  & \multicolumn{1}{c}{\texttt{ST-INR (L)}} & \multicolumn{1}{c}{\texttt{ST-INR(E)}} & \multicolumn{1}{c}{\texttt{SMF}} & \multicolumn{1}{c}{\texttt{GRALS}} & \multicolumn{1}{c}{\texttt{LCR}} & \multicolumn{1}{c}{\texttt{GLTL}} & \multicolumn{1}{c}{\texttt{FP}}  \\
    \midrule
    \multirow{3}[1]{*}{PEMS-BAY} 
          & RMSE (km/h)  & \textbf{6.99} & 9.49 & 8.74 & 9.86 & 8.44 & 11.45 &10.28  \\
          & WMAPE & \textbf{6.94\%} & 9.00\% & 9.17\%& 7.76\% & 8.92\% & 10.88\%  & 7.61\%   \\
          & MAE (km/h) & \textbf{4.34} & 5.63 & 5.74& 4.86 & 5.59 & 6.81 &4.77  \\
    \midrule
    \multirow{3}[1]{*}{Seattle} 
          & RMSE (km/h) & 8.51 & \textbf{7.14} & 9.36 & 7.57 & 8.02 & 8.82& 7.64 \\
          & WMAPE & 9.47\% &  \textbf{7.81\%} & 11.41\% & 8.43\% & 9.67\% & 11.26\% & 8.23\% \\
          & MAE (km/h)&  5.46&  \textbf{4.49} & 6.58 & 4.86 & 5.57 &  6.49 & 4.73 \\
    \bottomrule
    \multicolumn{5}{l}{\scriptsize{Best performances are bold marked.}}
    \end{tabular}%
  \label{Tab: kriging_results}%
\end{table}%

\begin{figure}[!htb]
\centering
\subfigure[Seattle data]{
\includegraphics[scale=0.3]{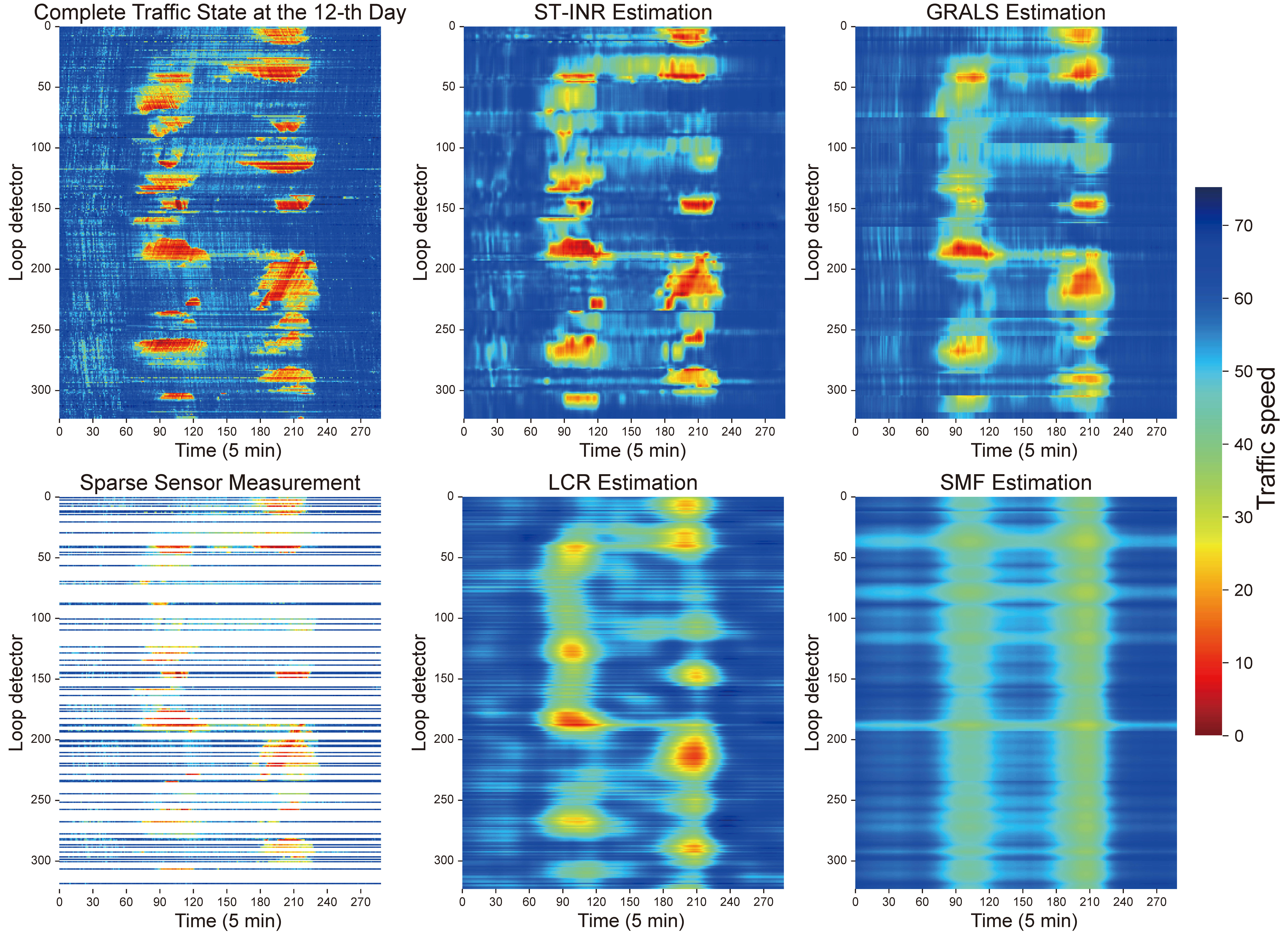}
}
\centering
\subfigure[PeMS data]{
\includegraphics[scale=0.35]{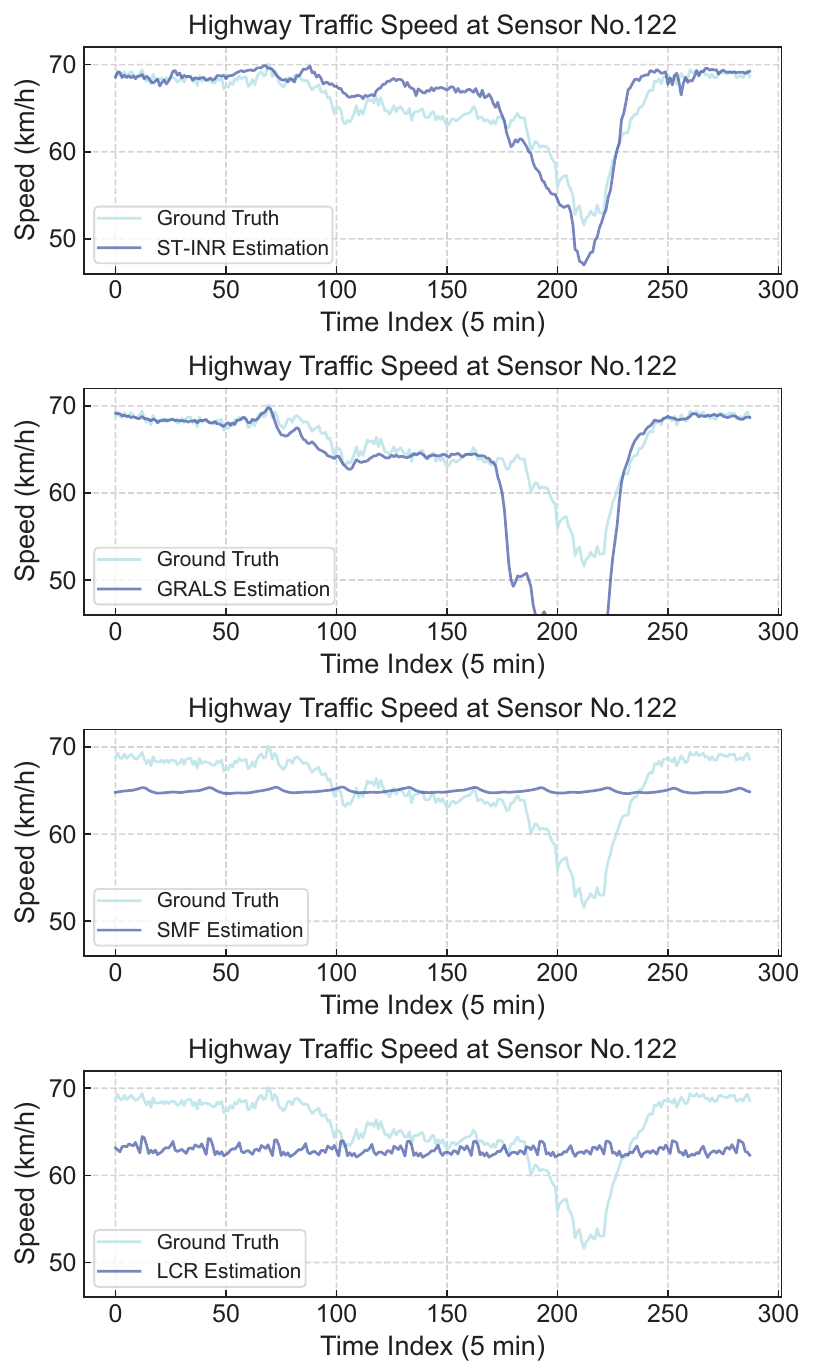}
}
\caption{Model performances on Seattle and PEMS-BAY speed data.}
\label{Fig:highway_result}
\end{figure}

It is demonstrated that ST-INR can be generalized to different coordinate systems (input spaces). For PEMS-BAY data with more complex graph structures, the Laplacian embedding plays a more important role in refining the input coordinates. In addition, because of the inherent smoothness property derived from the MLP, ST-INR can generate smooth estimations on both regular lattice and spectral embedding, even without any explicit smoothness or graph regularization terms. This finding echoes our analysis in Section \ref{Sec: smoothness}.
Combined with the excellent capability of encoding high-frequency patterns, ST-INR can produce more realistic reconstruction than other baselines. Conversely, smoothness regularization on Euclidean space does not achieve good performance on non-Euclidean space, such as the results of LCR and SMF in Fig. \ref{Fig:highway_result} (b).

\begin{figure}[!htb]
\centering
\includegraphics[scale=0.3]{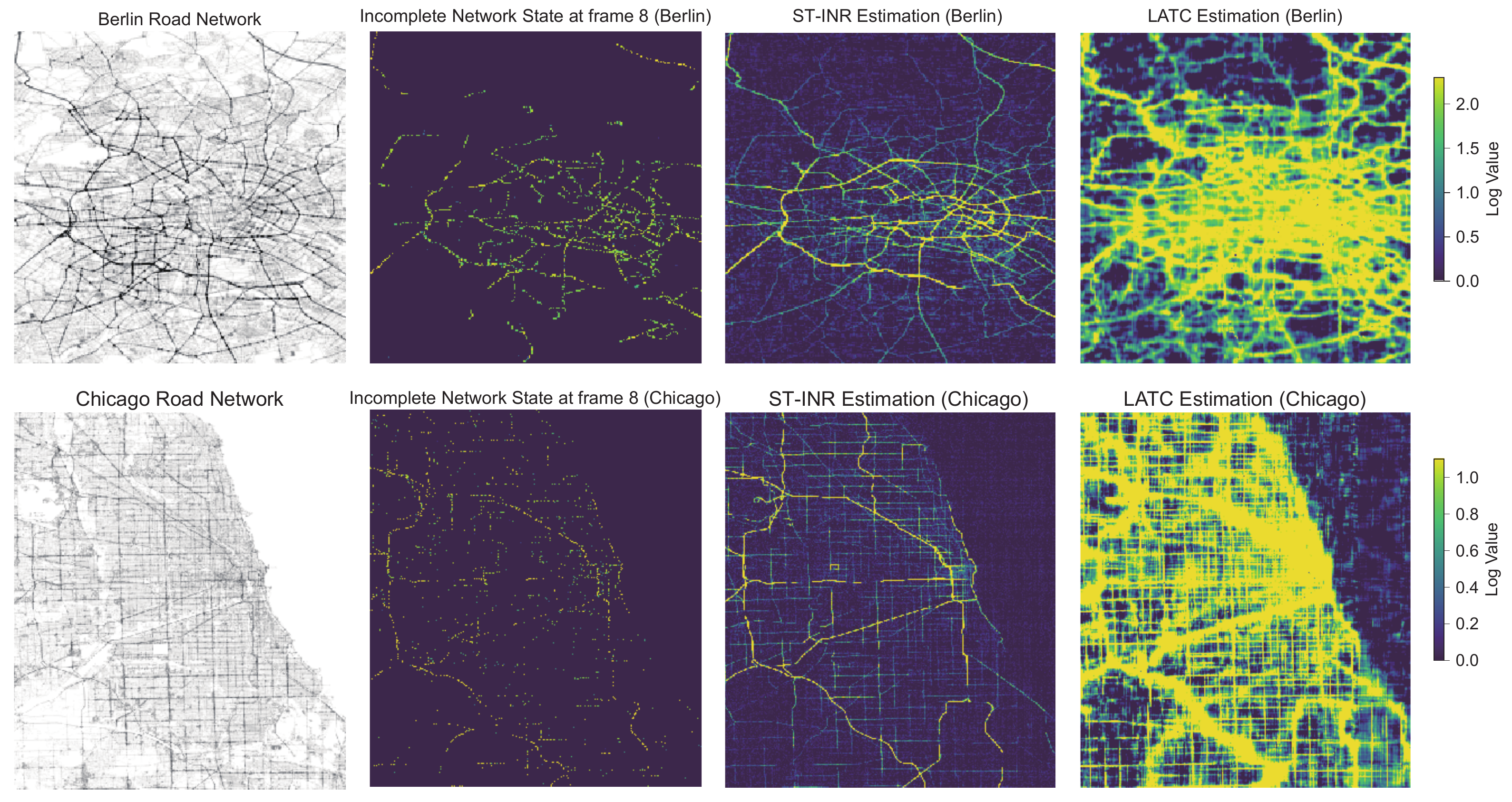}
\caption{Model performances on Chicago and Berlin network.}
\label{fig:network_result_urban}
\end{figure}

The case of urban network is more challenging due to the complexity of urban road network. As a result, we randomly select $50\%$ of the raw data as observations. 
% \textcolor{gray}{Recall that ST-INR can efficiently adapt to instance-specific patterns due to the autodecoding mechanism. We can train a single ST-INR to represent two different networks at once.
% To compare, we train a single INR on both two without using meta-learning. We also train LATC models on each of the Berlin and Chicago networks and report the results. } 
This tasks requires the model to learn high-frequency patterns to reconstruct the complexity of urban road network. To evaluate, we train ST-INR are each of the Berlin and Chicago networks and compare it with state-of-the-art tensor completion model LATC.
Fig. \ref{fig:network_result_urban} reveals some surprising findings. First, ST-INR can successfully recover the complex network topology, as well as the traffic states, even without the guidance of the network topology. 
% \textcolor{gray}{Second, a single ST-INR can learn to predict two distinct network patterns. In contrast, it would mismatch two road network structures without instance-wise decoding.} 
Second, pure low-rank models are incapable of learning these data with complicated topological structures, which again verifies the significance of high-frequency components.
Our model encodes a wide range of high-frequency series in the input layer, thereby being adaptive to different frequency patterns in real-world traffic data.

\section{Algorithmic analysis}\label{Sec: Algorithmic analysis}
On top of the analysis in Section \ref{Sec: theoretical_analysis} and the experimental results in Section \ref{Sec: experiments}, this section provides more discussions on the properties of our model to examine how these properties contribute to the learning of STTD.

{\subsection{Architectural ablation study}
We first assess the impact of each proposed modeling scheme by performing ablation studies. Three aspects are evaluated: (1) The effectiveness of the factorization scheme in Eq. \eqref{eq:factorize}; (2) The impact of low-rank constraints on the middle transform matrix $\mathbf{M}_{xt}$; (3) The significance of graph spetral embedding in Eq. \eqref{eq:graph_embedding}. Results in Tab. \ref{ablation} demonstrate the effectiveness of each modular design.

\begin{table}[!htb]
  \centering
  \caption{{Ablation studies on the model architecture.}}
    \begin{tabular}{l|ccccc}
    \toprule
          Dataset & Variant & \multicolumn{1}{l}{w/o factorization} & \multicolumn{1}{l}{w low-rank $\mathbf{M}_{xt}$} & \multicolumn{1}{l}{w/o Laplace}   & \multicolumn{1}{l}{Full}\\
    \midrule
    \multirow{2}{*}{\rotatebox{0}{TaxiNYC}} &  WMAPE & 18.92\% & 17.44\% & N/A & \textbf{14.53\%} \\
    & RMSE & 13.13 & 11.08 & N/A & \textbf{8.71}\\
    \midrule
    \multirow{2}{*}{\rotatebox{0}{PEMS-BAY}} &  WMAPE & 9.48\% & 8.02\% & 9.10\% & \textbf{6.94\%} \\
    & RMSE & 9.90 & 8.05 &9.55  & \textbf{6.99} \\
    \bottomrule
    \end{tabular}%
  \label{ablation}%
\end{table}%

}

\subsection{Frequency analysis}\label{sec:frequency_analysis}
Compared to vanilla MLPs and low-rank models, ST-INR can capture high-frequency patterns within natural signals. In other words, the incorporation of high-frequency information benefits the learning of ST-INR. We indicate this property by examining the Fourier spectrum and the convergence behaviors during training.
The Fourier spectrum in Fig. \ref{fig:frequency_analysis} (a) shows that ST-INR can utilize high-frequency encodings to model meaningful frequency details in real data. Compared to pure low-rank matrix factorization, it preserves more energy in the back part of the spectrum.
Fig. \ref{fig:frequency_analysis} (b) compares the loss curves under different frequency settings. After gradually increasing the frequency encoding from 0 to 10, the model achieves a lower training loss, which means that it becomes easier to fit the real-world data with higher frequencies. 

\begin{figure}[!htb]
\centering
\subfigure[Single-sided frequency spectrum on Seattle data]{
\includegraphics[scale=0.45]{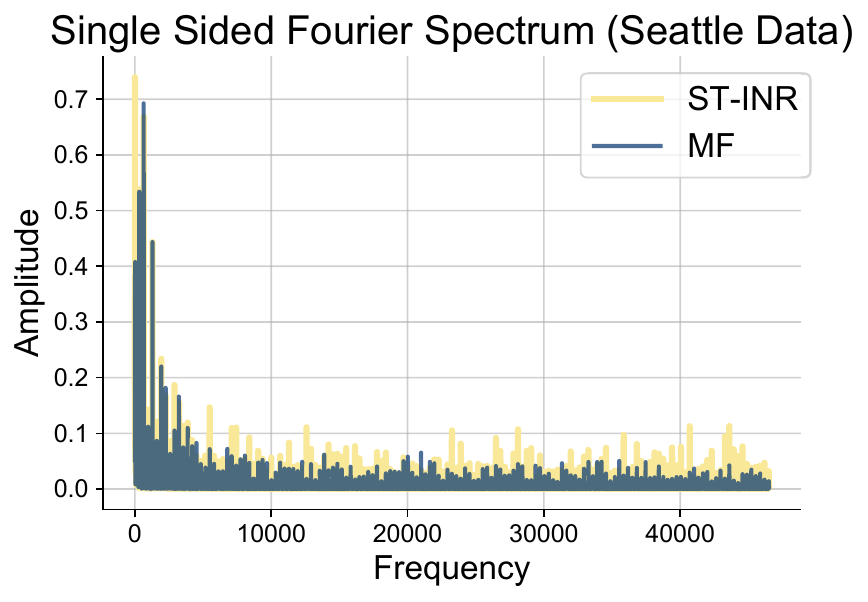}
}
\centering
\subfigure[Training loss under different frequency components (three different runs)]{
\includegraphics[scale=0.4]{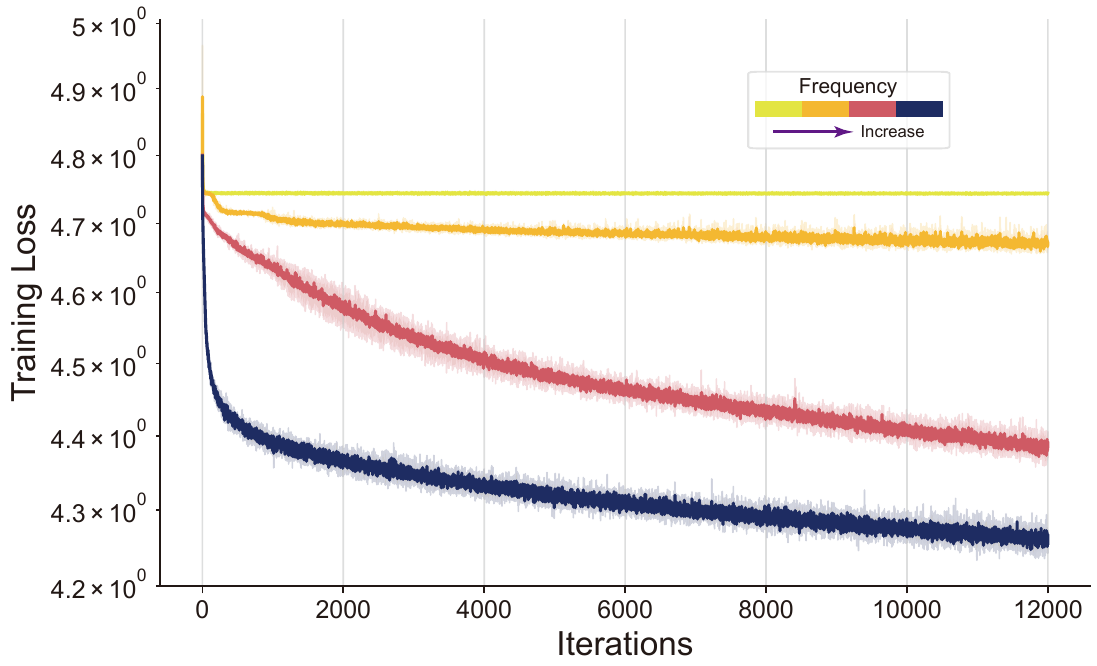}
}
\caption{Frequency analysis.}
\label{fig:frequency_analysis}
\end{figure}

\begin{figure}[!htb]
\centering
\subfigure[Fourier features with different parameters]{
\includegraphics[scale=0.3]{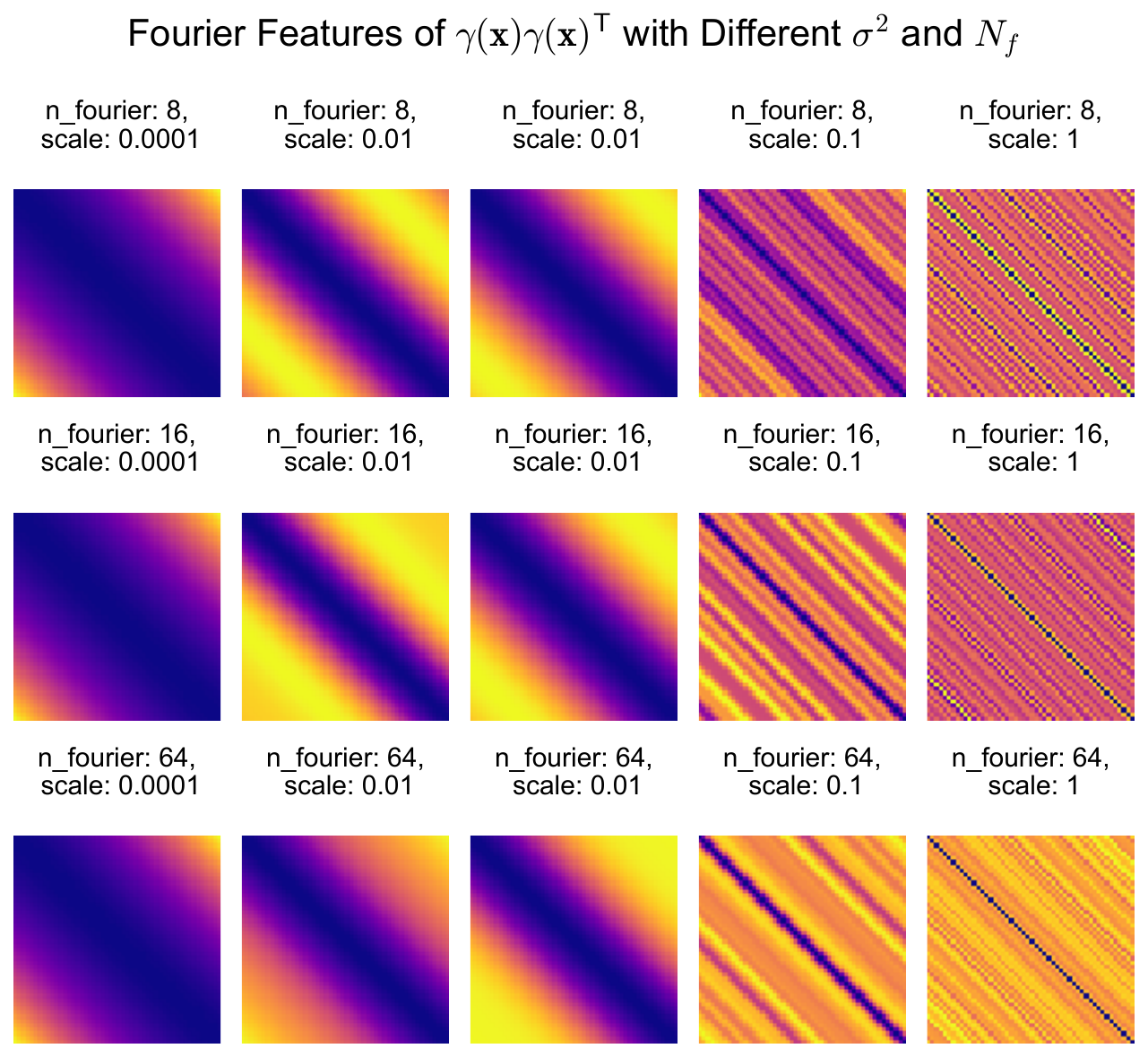}}
\centering
\subfigure[Factorized frequency with different spatial and temporal modes]{
\includegraphics[scale=0.4]{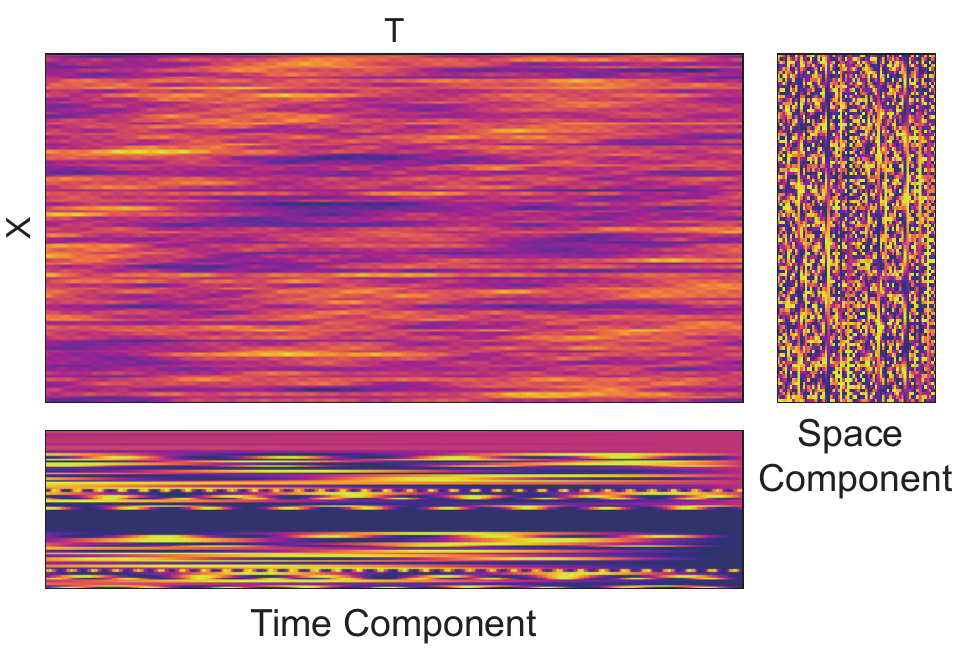}}
\caption{Visualization of random Fourier features.}
\label{fig:frequency_factor}
\end{figure}

Note that the CRF in Eq. \eqref{eq:crf} has two important parameters, i.e., the number of Fourier features $N_f$ and the dictionary of scales $\{\sigma_1^2,\dots,\sigma_{N_f}^2\}$. We study its effect in Fig. \ref{fig:frequency_factor}. As can be seen, increasing the scales can have a wider spectrum, covering more frequency patterns, such as periodicity. All the while, adding the number of CRF can interpolate the frequency distributions, sharing more details. Moreover, thanks to the factorized design, we can use separate frequency encodings with different $N_f$ and $\sigma^2$ along the time and space dimensions.

{To further justify our hypothesis and indicate the impact of high-frequency encoding, we perform ablation studies to examine the model performances after removing the high-frequency encoding, i.e., Eqs. \eqref{eq:siren} and \eqref{eq:crf}. To ensure a fair comparison, we directly set the sine activation to ReLU and replace the CRF by a linear layer with the same dimension. Results are shown in Tab. \ref{Tab: ablation_frequency}.

\begin{table}[!htb]
  \centering
  \caption{{Results (in terms of WMAPE and RMSE) of ST-INR w/ and w/o high-frequency encoding.}}
    \begin{tabular}{c|c|ccccc}
    \toprule
    \multicolumn{2}{c}{Dataset}  & \multicolumn{1}{c}{\texttt{NGSIM}}& \multicolumn{1}{c}{\texttt{TaxiBJ64}}  & \multicolumn{1}{c}{\texttt{TaxiNYC}} & \multicolumn{1}{c}{\texttt{PEMS-BAY}} & \multicolumn{1}{c}{\texttt{Seattle}} \\
    \midrule
    \multirow{2}{*}{\rotatebox{0}{w/ high-frequency enc.}} & WMAPE & \textbf{11.05\%} &\textbf{17.44\%} & \textbf{14.53\%} & \textbf{6.94\%} &\textbf{7.81\%} \\
    & RMSE   & \textbf{4.88} & \textbf{21.64} & \textbf{8.71}& \textbf{6.99} & \textbf{4.49}\\
    \midrule
    \multirow{2}{*}{w/o high-frequency enc.} & WMAPE & {32.33\%} & 75.92\% & 95.15\% & 11.20\% & 12.78\%  \\
    & RMSE   & {13.66} & 99.32& 76.16& 10.31 & 11.54 \\
    \bottomrule
    \end{tabular}%
  \label{Tab: ablation_frequency}%
\end{table}%

As can be seen, ST-INR struggles to achieve desirable performances without explicit high-frequency encoding. For dynamic mobility flow data such as grid-based flow and OD flow, the performance degradation is more severe. These results verify the significance of learning high-frequency components in STTD modeling.
}

\subsection{Model efficiency}\label{exp:efficiency}
Since our model is optimized through stochastic gradient descent, it can be efficiently implemented by a modern deep learning framework such as PyTorch. Besides, due to the spatial-temporal factorization, it features reduced computational complexity. We adopt TaxiBJ data with different data scales (from $8\times8$ to $128\times128$) to examine its efficiency in Fig. \ref{Fig:efficiency}. All measurements are performed on a single GPU device.
As indicated by our analysis in Section \ref{sec:complexity}, the full mode has cubic complexity with respect to the spatial (temporal) dimension $n$. In contrast, the factorized version approximates to scale linearly to the dimensions, which is consistent with the results in Fig. \ref{Fig:efficiency}. This scalability makes our model feasible for large-scale datasets.

\begin{figure}[!htb]
\centering
\includegraphics[scale=0.5]{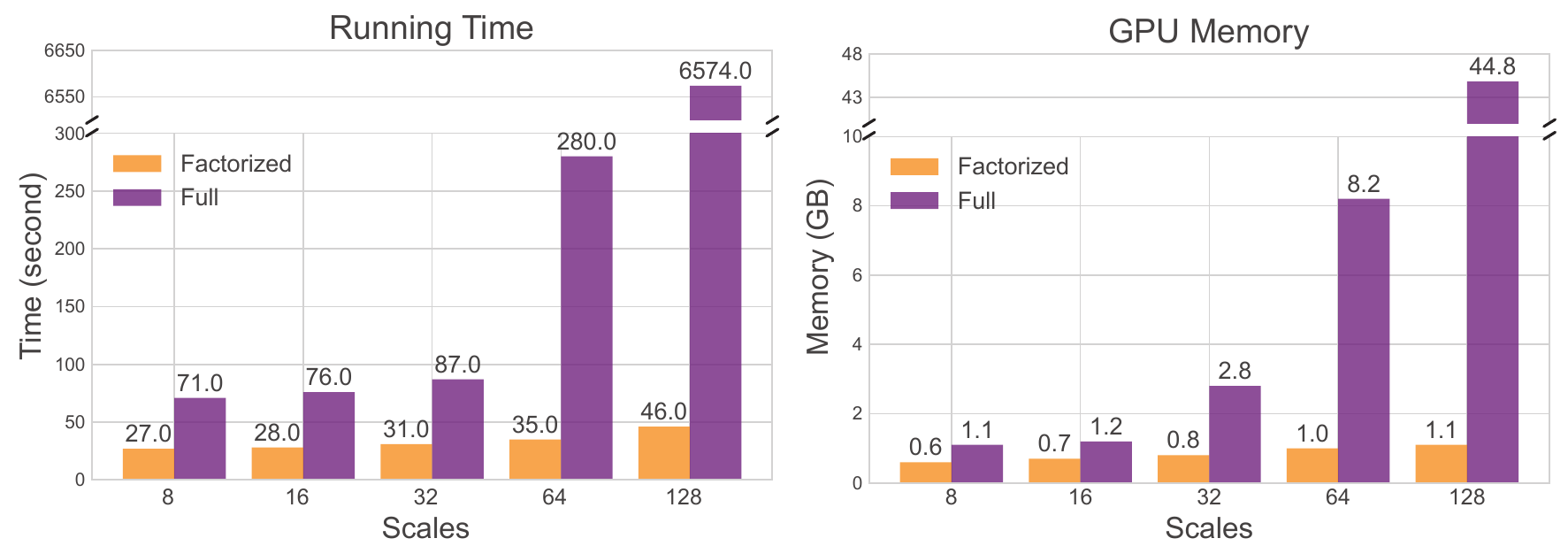}
\caption{Model efficiency on TaxiBJ data.}
\label{Fig:efficiency}
\end{figure}

\subsection{Hyperparameter analysis}\label{sec:hyperparameter}
As a deep neural network model, it has a number of parameters that need to be tuned. As suggested by \citet{luo2023low}, the weight decay parameter has an impact on overall continuity, thus affecting the smoothness regularization. Furthermore, deep neural networks are supposed to benefit from the model depth \citep{he2016deep}. We evaluate the two hyperparameters in several adopted datasets in Fig. \ref{Fig:hyperparameter}. We can see that adding more INR layers indeed benefits the model performance, but increases the risks of overfitting on small datasets. 
For data with more perceivable details such as Seattle, using a too large smoothness penalty will make the model degrade. For data with relatively simpler patterns, such as TaxiBJ, more smooth solutions are encouraged.
In practice, setting the decay parameter to 1 can yield a desirable performance on most of the data.

\begin{figure}[!htb]
\centering
\includegraphics[scale=0.4]{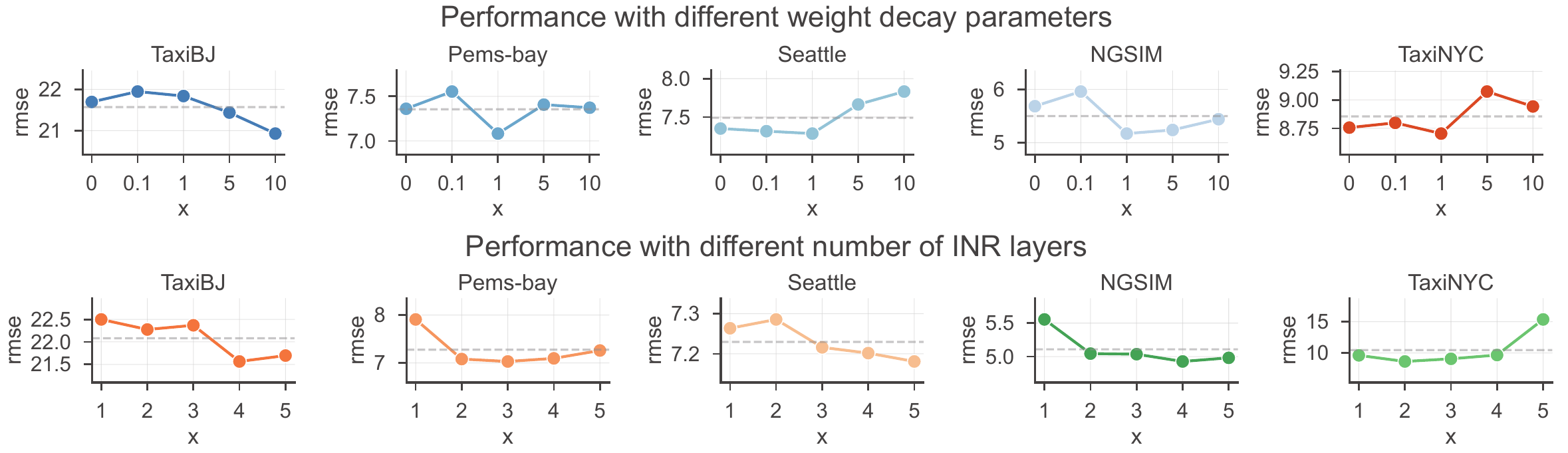}
\caption{Hyperparameter studies on different datasets.}
\label{Fig:hyperparameter}
\end{figure}

\subsection{Implicit low-rank regularization}\label{exp:low-rank}
\begin{figure}[!htb]
\centering
\subfigure[Nuclear norm under different missing rates]{
\includegraphics[scale=0.42]{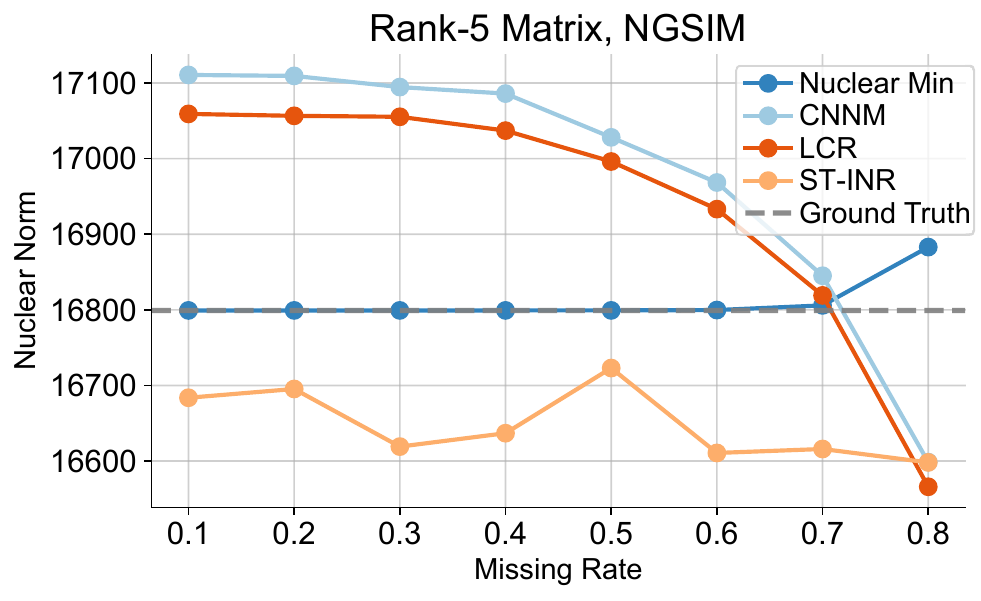}}
\centering
\subfigure[Effective rank under different missing rates]{
\includegraphics[scale=0.42]{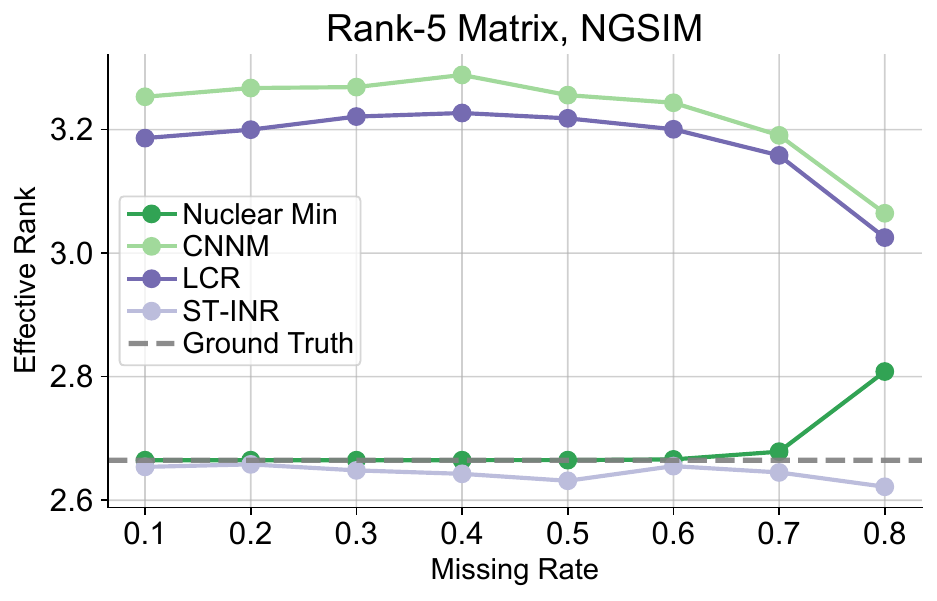}}
\centering
\subfigure[Nuclear norm with different factorized dimension]{
\includegraphics[scale=0.42]{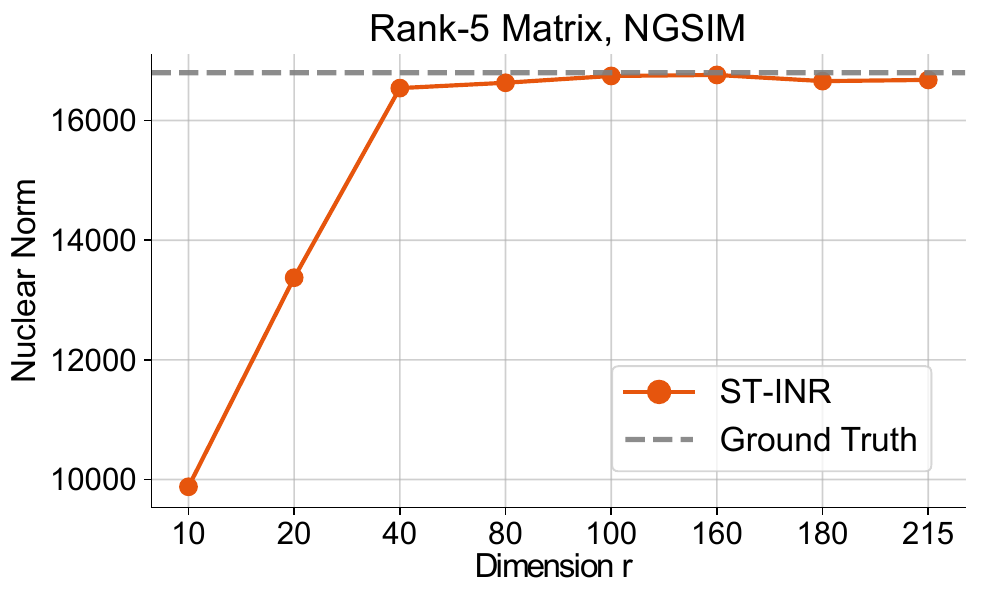}}
\centering
\subfigure[Effective rank with different factorized dimension]{
\includegraphics[scale=0.42]{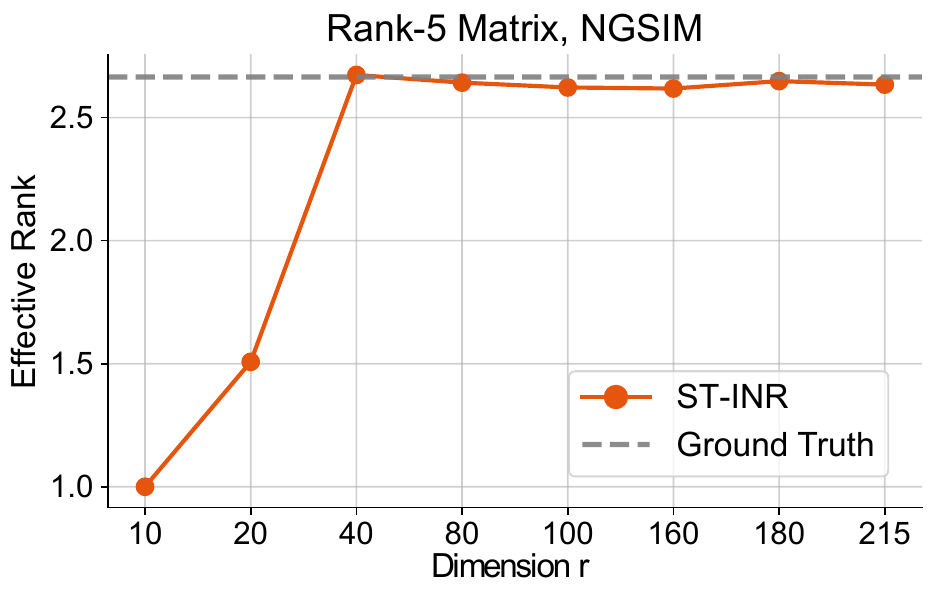}}
\caption{Implicit low-rank regularization on a rank-5 speed matrix of NGSIM data.}
\label{fig:low-rankness}
\end{figure}

Low-rankness is a strong inductive bias for STTD, especially when the observations are sparse.
The implicit low-rankness of our model is discussed in Section \ref{sec:implicit_low_rank}. We now provide some empirical evidence on this implicit regularization. To evaluate the rank (norm) minimization effect, we manually construct a rank-5 matrix by reconstructing NGSIM data using a truncated singular value decomposition, i.e., $\widetilde{\mathbf{X}}=\mathbf{U}[:,:5]\mathbf{D}[:5,:5]\mathbf{V}^{\mathsf{T}}[:5,:]$. We then examine the nuclear norm and effective rank \citep{roy2007effective} of the estimations of different methods under different missing rates.
As a reference, we add the result of a rank-constrained minimum nuclear norm as the global optima, termed ``Nuclear Min'' in Fig. \ref{fig:low-rankness} (a) and (b).

The results presented in Fig. \ref{fig:low-rankness} (a)-(b) indicate that as the missing rate increases, all methods converge towards a solution with lower rank and smaller nuclear norm. However, our model consistently achieves a lower norm and a smaller effective rank than other matrix completion methods. More importantly, although our model has much smaller nuclear norm than ground truth and global optimal values, it has close effective rank values to them. This reveals that our model does not rely on the minimization of a surrogate nuclear norm, but directly encourages a low-rank solution. This differentiates it from existing nuclear norm-based methods.

To further evaluate whether the dimension of factorization has impacts on the implicit low-rankness like the rank constraint on factor matrices of the MF model, we increase the dimension from $r=10$ to $r=\min\{N,T\}$. $r=\min\{N,T\}$ results in a full-dimensional factorization. The results in Fig. \ref{fig:low-rankness} (c)-(d) are insightful: a low rank-constrained model fails to recover the true norm and rank values, leading to a underfitting resolution; Instead, with the gradual relaxation of the rank constraints, even the full-dimensional case can produce a precise reconstruction.  
This finding is in alignment with previous work \citep{gunasekar2017implicit}.

To summarize, the above empirical results suggest that our model has an implicit regularization that admits stronger bias towards low-rankness, which shows better performance than existing nuclear norm minimization methods. 
A desirable reconstruction can be achieved by using the implicit regularization of DMF in Section \ref{sec:implicit_low_rank}, without the need to elaborate on nuclear norm-based surrogate or perform rank-constrained alternating optimization.
In addition, we can also bypass the need to tune the rank parameter $r$ by directly using a full-dimensional factorization.
These findings could inspire a new line for traffic data imputation or forecasting studies that adopt a newly designed norm or a variation of the nuclear norm optimized by ADMM or ALS.

\section{Conclusion and outlook}\label{conclusions}
In this work, we demonstrate a novel {STTD} learning method based on implicit neural representations, termed ST-INR. By parameterizing STTD as continuous deep neural networks, we train ST-INR to directly map the spatial-temporal coordinates to the traffic states. 
ST-INR explicitly encodes high-frequency structures to learn complex details of STTD while at the same time implicitly learning low-rank priors and smoothness regularization from data itself to reconstruct dominant patterns.
Due to the generality of this representation, it can be exploited to model a variety of STTD, such as vehicle trajectory, origin-destination flows, grid flows, highway sensor networks, and urban networks. 
% \textcolor{gray}{By virtue of the meta-learning paradigm, it is generalizable to a series of data instances, which overcome one of the main bottlenecks of existing optimization-based models.}
Experimental results on various real-world benchmarks indicate that our model consistently outperforms traditional low-rank models such as tensor factorization and nuclear norm minimization. It also has the potential to generalize across different data structures and problem settings. In addition, other important properties such as the incorporation of high-frequency structures, implicit low-rankness, and inherent smoothness can function as new inductive biases for STTD modeling, providing an alternative to conventional low-rank models.

Meanwhile, based on our framework, there are still some directions that require future efforts, for example,
\begin{enumerate}
    \item Physics-informed learning of traffic dynamics: How to enable the current framework to be guided by physics? A possibly feasible routine in integrating it with physics-based machine learning methods, such as physics-informed neural networks \citep{raissi2019physics}.
    \item Forecasting irregularly sampled traffic time series: Real-world traffic data may be recorded by sensors with different sampling frequencies or by varying the number of mobile sensors. Due to the continuous nature of our model, it is possible to adopt it to forecast irregular traffic data \citep{ye2012short}. 
    \item Representation learning of vehicle paths: How to organize vehicle path sets into an acceptable input format of INRs and perform path flow estimation for urban road networks is a promising question. 
    \item {Integration with large language models: ST-INR can be adopted as a decoder for different encoder architectures. For example, it can be integrated into the downstream processing of large language models to deal with more complicated traffic problems \citep{liu2023can,qu2023envisioning}.}
\end{enumerate}

Overall, we believe that the proposed ST-INR can provide an opportunity to unify STTD learning methods and thereby facilitate the development of a large foundation model for generalized STTD analysis. 

% \appendix
% \section*{Appendix}

% \section{High-dimensional extensions}
% \label{Appendix: tensor}
% Draw inspiration from \citep{luo2023low}, we can organize the multidimensional data into a Tucker tensor \citep{kolda2009tensor} format. Different from the Tucker model defined on discrete grids, ST-INR are fully defined on continuous input domains. Specifically, given a data instance with $M$ input-output pairs $\mathbf{x}=\{(\mathbf{v}_i,\mathbf{y}_i)\}_{i=1}^M$ where $\mathbf{v}_i\in\mathbb{R}^{N}$ is the $N$ dimensional input coordinate and $\mathbf{y}_i\in\mathbb{R}^{c_{\text{out}}}$ is the true data value. The ST-INR model for high-dimensional data can be formulated as:
% \begin{equation}
% \begin{aligned}
%     &\min_{\Theta}\mathcal{L}(\Theta;\mathbf{x})=\frac{1}{M}\sum_{i=1}^M\Vert\mathbf{y}_i-\Phi_{\Theta}(\mathbf{v}_i) \Vert_2^2, \\
%     &\Phi_{\Theta}(\mathbf{v}_i) = \mathcal{M}\times_1\Phi_{\theta_1}(v_1)\times_2\cdots\times_N\Phi_{\theta_N}(v_N), \forall (v_1,v_2,\dots,v_N)\in\mathbf{v}_i,\\
%     &\Phi_{\theta_i}: v_i\mapsto \Phi_{\theta_i}(v_i)\in\mathbb{R}^{n_i},
% \end{aligned}
% \end{equation}
% where $\mathcal{M}\in\mathbb{R}^{n_1\times \dots n_N}$ is the core tensor, and $\Theta=\{\theta_1,\theta_2,\dots,\theta_N\}\cup\mathcal{M}$ are the model parameters.

\appendix
\section*{Appendix}
\section{More results on traffic state estimation}
{In this section, we provide more results on the highway traffic state estimation task in Fig. \ref{fig:ngsim_result_sup}. 
In comparison to the state-of-the-art tensor-based estimation model, STHTC, the proposed model is capable of reconstructing congestion waves with greater realism. Notably, the results obtained from STHTC exhibit some discontinuities and abnormal zero values.
}

\begin{figure}[!htb]
\centering
\includegraphics[width=0.99\textwidth]{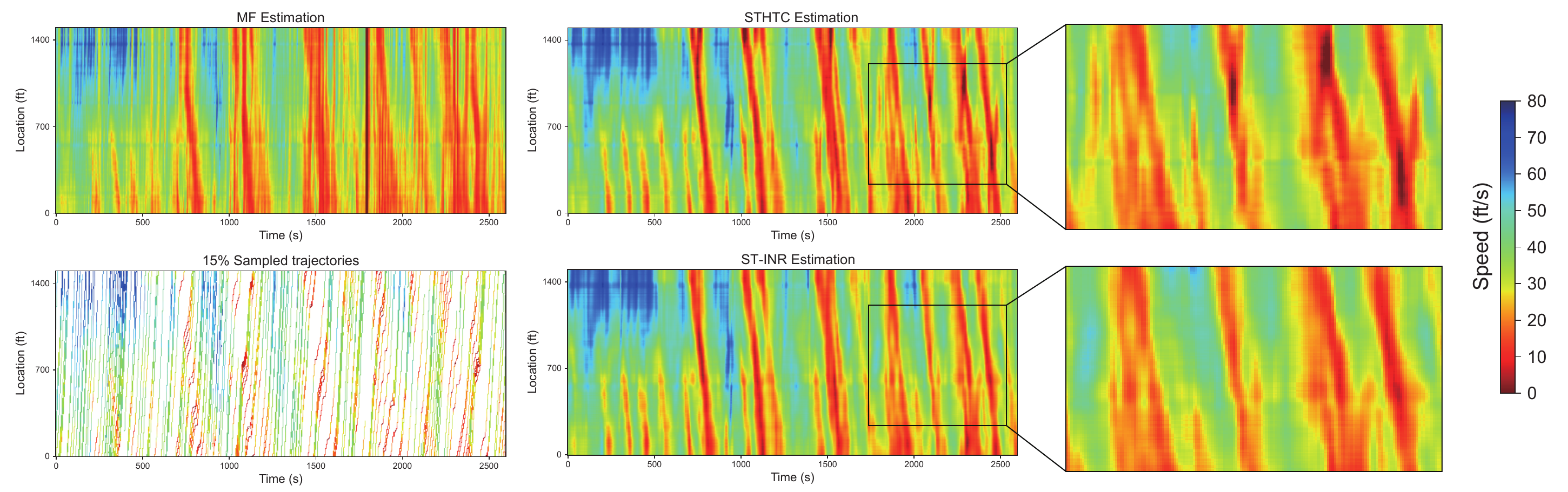}
\caption{{Supplementary results of TSE on discrete grid.}}
\label{fig:ngsim_result_sup}
\end{figure}

\nolinenumbers
% \newpage
% \bigskip
\section*{Acknowledgement}
This research was sponsored by the National Natural Science Foundation of China (52125208), the National Natural Science Foundation of China's Young Scientists Fund (52302413), the Science and Technology Commission of Shanghai Municipality (No. 22dz1203200), and the Research Grants Council of the Hong Kong Special Administrative Region, China (Project No. PolyU/25209221 and PolyU/15206322).

% , the Shanghai Municipal Science and Technology Major Project (No. 2021SHZDZX0100)

% \newpage
\footnotesize
\bibliographystyle{elsarticle-harv}
\bibliography{reference}

% \begin{thebibliography}{33}
% \expandafter\ifx\csname natexlab\endcsname\relax\def\natexlab#1{#1}\fi
% \expandafter\ifx\csname url\endcsname\relax
%   \def\url#1{\texttt{#1}}\fi
% \expandafter\ifx\csname urlprefix\endcsname\relax\def\urlprefix{URL }\fi

% \bibitem[{Asif et~al.(2016)Asif, Mitrovic, Dauwels, and
%   Jaillet}]{asif2016matrix}
% Asif, M.~T., Mitrovic, N., Dauwels, J., Jaillet, P., 2016. Matrix and tensor
%   based methods for missing data estimation in large traffic networks. IEEE
%   Transactions on intelligent transportation systems 17~(7), 1816--1825.
% \bibitem[{Bolte et~al.(2014)Bolte, Sabach, and Teboulle}]{bolte2014proximal}
% Bolte, J., Sabach, S., Teboulle, M., 2014. Proximal alternating linearized
%   minimization for nonconvex and nonsmooth problems. Mathematical Programming
%   146~(1), 459--494.
% \end{thebibliography}

\end{document}